\documentclass{article}

\usepackage{microtype}
\usepackage{graphicx}
\usepackage{subcaption}
\usepackage{booktabs} 

\usepackage{hyperref}
\usepackage[table]{xcolor}

\usepackage[preprint]{icml2026}

\usepackage{amsmath}
\usepackage{amssymb}
\usepackage{mathtools}
\usepackage{amsthm}

\usepackage[capitalize,noabbrev]{cleveref}

\theoremstyle{plain}
\newtheorem{theorem}{Theorem}[section]

\theoremstyle{definition}
\newtheorem{definition}[theorem]{Definition}

\theoremstyle{remark}

\DeclareMathOperator*{\smax}{smax}
\DeclareMathOperator*{\smin}{smin}
\DeclareMathOperator*{\argmin}{argmin}
\DeclareMathOperator*{\argmax}{argmax}

\usepackage{alphalph}

\usepackage[textsize=tiny]{todonotes}

\def\ddefloop#1{\ifx\ddefloop#1\else\ddef{#1}\expandafter\ddefloop\fi}

\usepackage{tikz}

\def\ddef#1{\expandafter\def\csname v#1\endcsname{\ensuremath{\boldsymbol{#1}}}}
\ddefloop ABCDEFGHIJKLMNOPQRSTUVWXYZabcdefghijklmnopqrstuvwxyz\ddefloop

\def\ddef#1{\expandafter\def\csname v#1\endcsname{\ensuremath{\boldsymbol{\csname #1\endcsname}}}}
\ddefloop {alpha}{beta}{gamma}{delta}{epsilon}{varepsilon}{zeta}{eta}{theta}{vartheta}{iota}{kappa}{lambda}{mu}{nu}{xi}{pi}{varpi}{rho}{varrho}{sigma}{varsigma}{tau}{upsilon}{phi}{varphi}{chi}{psi}{omega}{Gamma}{Delta}{Theta}{Lambda}{Xi}{Pi}{Sigma}{varSigma}{Upsilon}{Phi}{Psi}{Omega}{ell}\ddefloop

\def\ddef#1{\expandafter\def\csname bb#1\endcsname{\ensuremath{\mathbb{#1}}}}
\ddefloop ABCDEFGHIJKLMNOPQRSTUVWXYZ\ddefloop

\begin{document}

\twocolumn[
  \icmltitle{Quality-Diversity Optimization as Multi-Objective Optimization}



  \icmlsetsymbol{equal}{*}

\begin{icmlauthorlist}
\icmlauthor{Xi Lin}{equal,xjtu}
\icmlauthor{Ping Guo}{equal,cityu}
\icmlauthor{Yilu Liu}{cityu}
\icmlauthor{Qingfu Zhang}{cityu}
\icmlauthor{Jianyong Sun}{xjtu}
\end{icmlauthorlist}

\icmlaffiliation{xjtu}{Xi'an Jiaotong University}
\icmlaffiliation{cityu}{City University of Hong Kong}

\icmlcorrespondingauthor{Xi Lin}{xi.lin@xjtu.edu.cn}
\icmlcorrespondingauthor{Jianyong Sun}{jy.sun@xjtu.edu.cn}

  \icmlkeywords{Quality-Diversity Optimization, Multi-Objective Optimization}

  \vskip 0.3in
]



\printAffiliationsAndNotice{\icmlEqualContribution}

\begin{abstract}

The Quality-Diversity (QD) optimization aims to discover a collection of high-performing solutions that simultaneously exhibit diverse behaviors within a user-defined behavior space. This paradigm has stimulated significant research interest and demonstrated practical utility in domains including robot control, creative design, and adversarial sample generation. A variety of QD algorithms with distinct design principles have been proposed in recent years. Instead of proposing a new QD algorithm, this work introduces a novel reformulation by casting the QD optimization as a multi-objective optimization (MOO) problem with a huge number of optimization objectives. By establishing this connection, we enable the direct adoption of well-established MOO methods, particularly set-based scalarization techniques, to solve QD problems through a collaborative search process. We further provide a theoretical analysis demonstrating that our approach inherits theoretical guarantees from MOO while providing desirable properties for the QD optimization. Experimental studies across several QD applications confirm that our method achieves performance competitive with state-of-the-art QD algorithms.

\end{abstract}

\section{Introduction}

While classic optimization typically seeks a single best solution, practical problems often require a range of high-performing alternatives with distinct characteristics. To address this need, Quality-Diversity (QD) optimization aims to find a set of high-quality and diverse solutions that collectively cover a user-defined behavior space, thereby representing the full spectrum of possible solution types \cite{pugh2016quality, chatzilygeroudis2021quality}. This optimization paradigm has demonstrated broad applicability across domains such as creative image composition~\cite{secretan2011picbreeder, nguyen2016understanding, tian2022modern}, content generation~\cite{bradley2023quality, ding2024quality}, robot control~\cite{doncieux2015evolutionary, cully2015robots, wan2025diversifying}, game agent design~\cite{ecoffet2019go, ecoffet2021first}, scenario generation~\cite{gonzalez2020finding, bhatt2022deep, fontaine2022evaluating, zhang2023arbitrarily}, adversarial sample generation for images~\cite{nguyen2015deep} or large language models~\cite{samvelyan2024rainbow, wang2025quality}, and automatic discovery for crystal structure~\cite{janmohamed2024multi} or algorithm design~\cite{novikov2025alphaevolve}.

Building upon the core QD optimization objective, numerous methods have been developed in recent years. Pioneering works such as Novelty Search with Local Competition (NSLC)~\cite{lehman2011abandoning, lehman2011evolving} and Multi-dimensional Archive of Phenotypic Elites (MAP-Elites)~\cite{mouret2015illuminating} established the core QD paradigm for balancing quality and diversity within a behavior space. Subsequent work has extended QD to more complex and realistic scenarios, including high-dimensional behavior spaces~\cite{vassiliades2017using}, expensive evaluations~\cite{gaier2017data, zhang2022deep}, multi-task optimization~\cite{mouret2020quality}, unknown behavior descriptions~\cite{hedayatian2025autoqd}, and gradient-aware optimization~\cite{nilsson2021policy, fontaine2021differentiable,fontaine2023covariance,hedayatian2025soft}.

\begin{figure*}[t]
\centering
\subfloat[MAP-Elites]{\includegraphics[width = 0.25\linewidth]{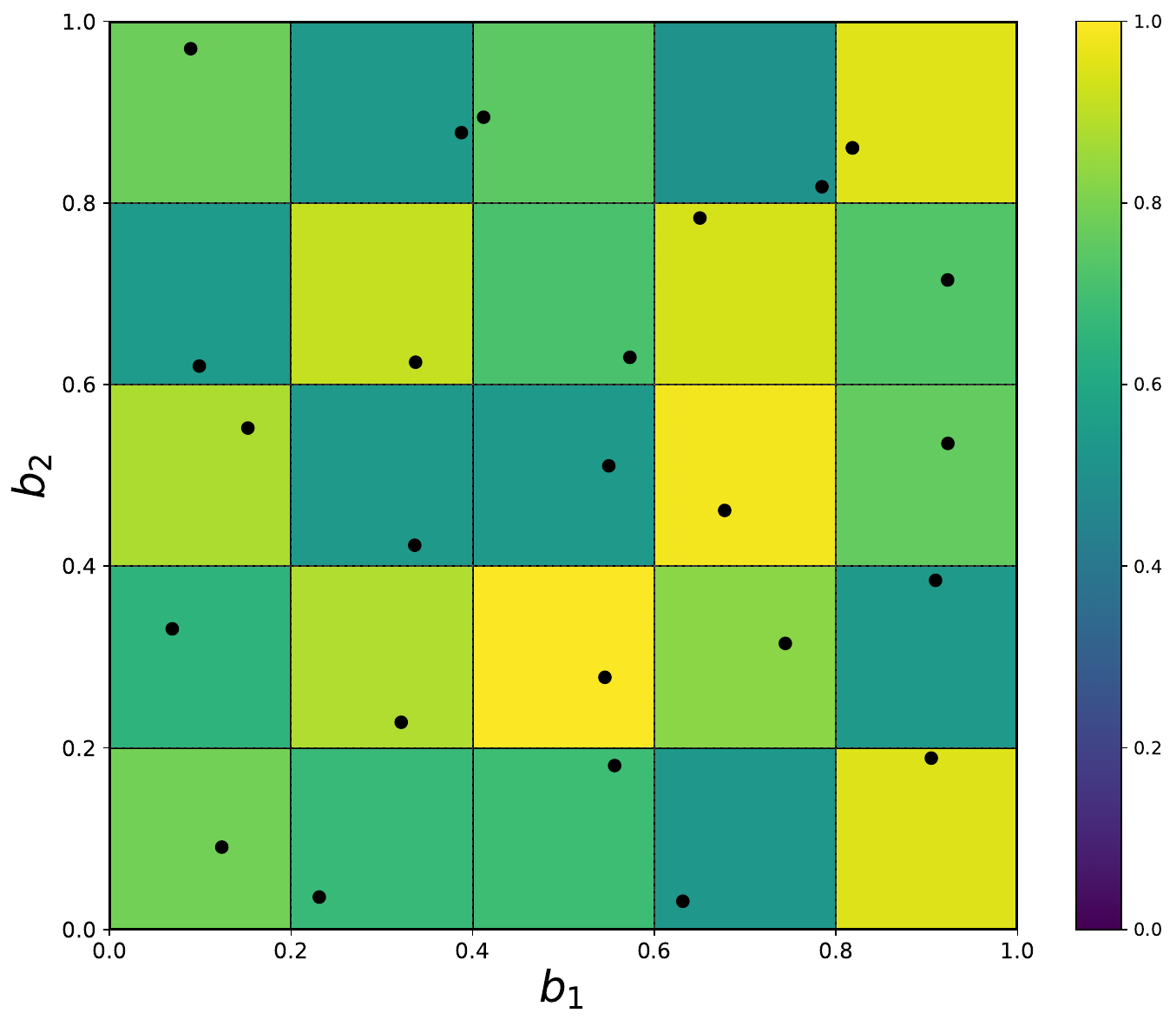}}\hfill
\subfloat[CVT-MAP-Elites]{\includegraphics[width = 0.25\linewidth]{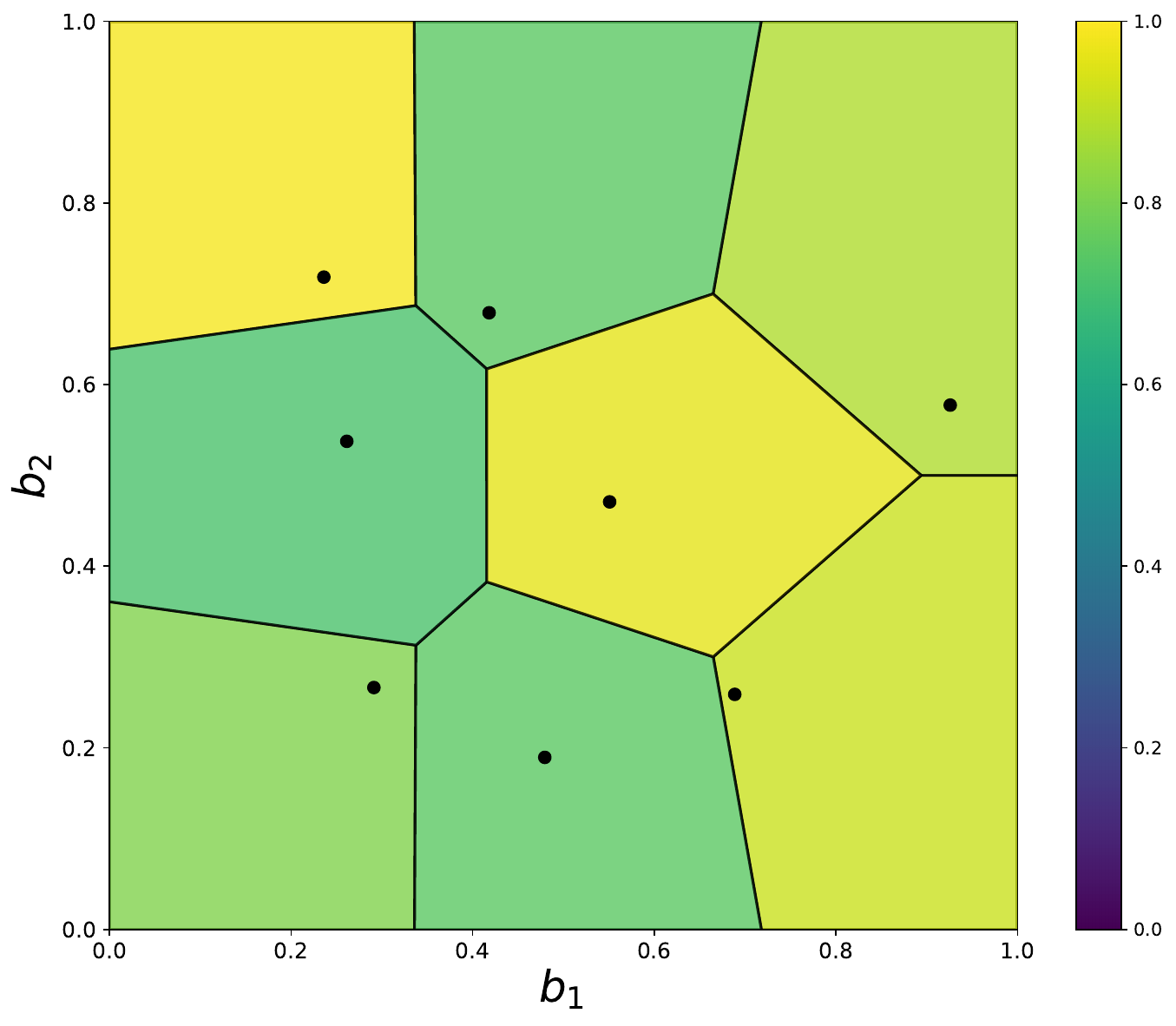}}\hfill
\subfloat[Soft QD]{\includegraphics[width = 0.25\linewidth]{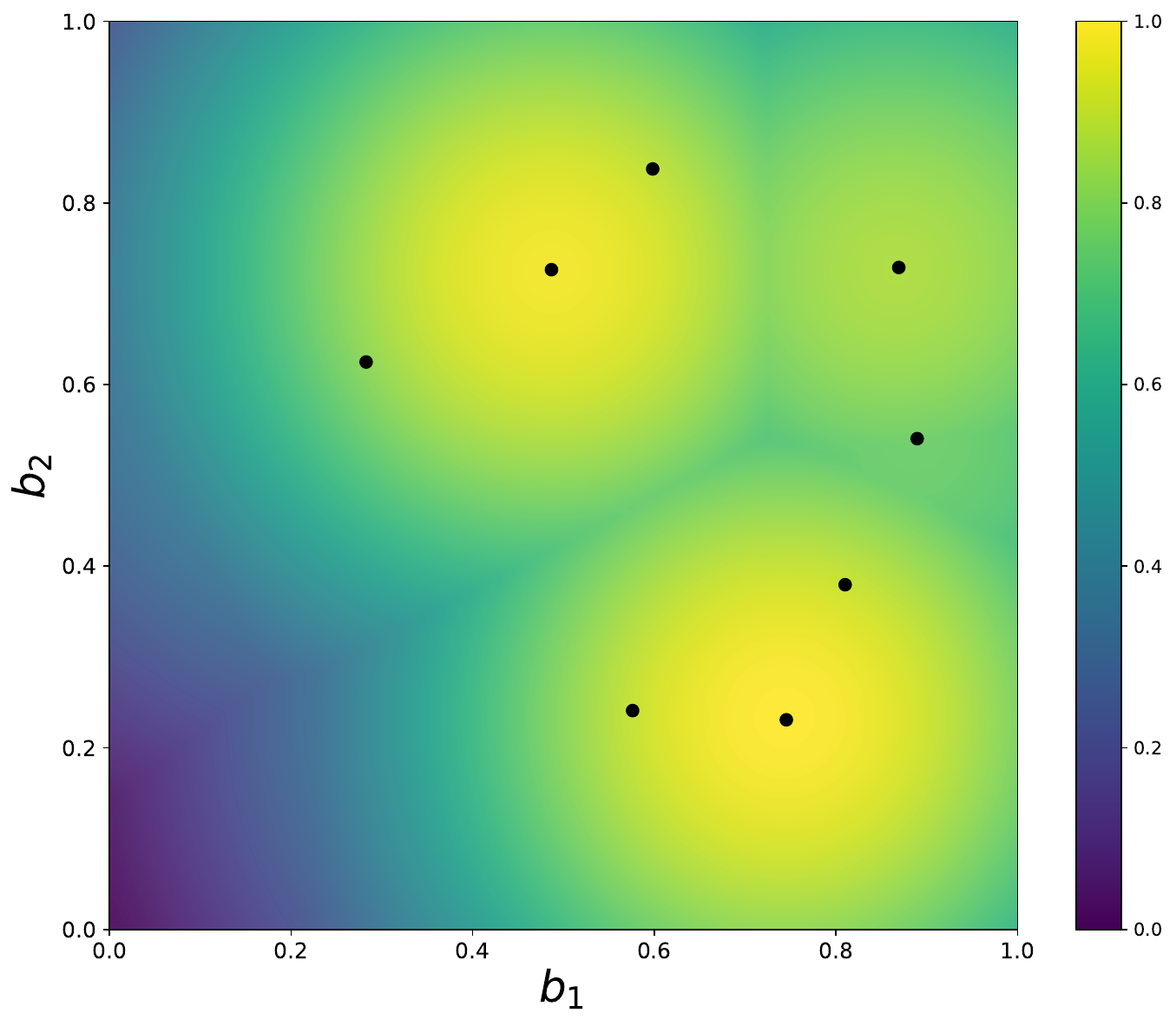}}
\hfill
\subfloat[QD as MOO (Ours)]{\includegraphics[width = 0.25\linewidth]{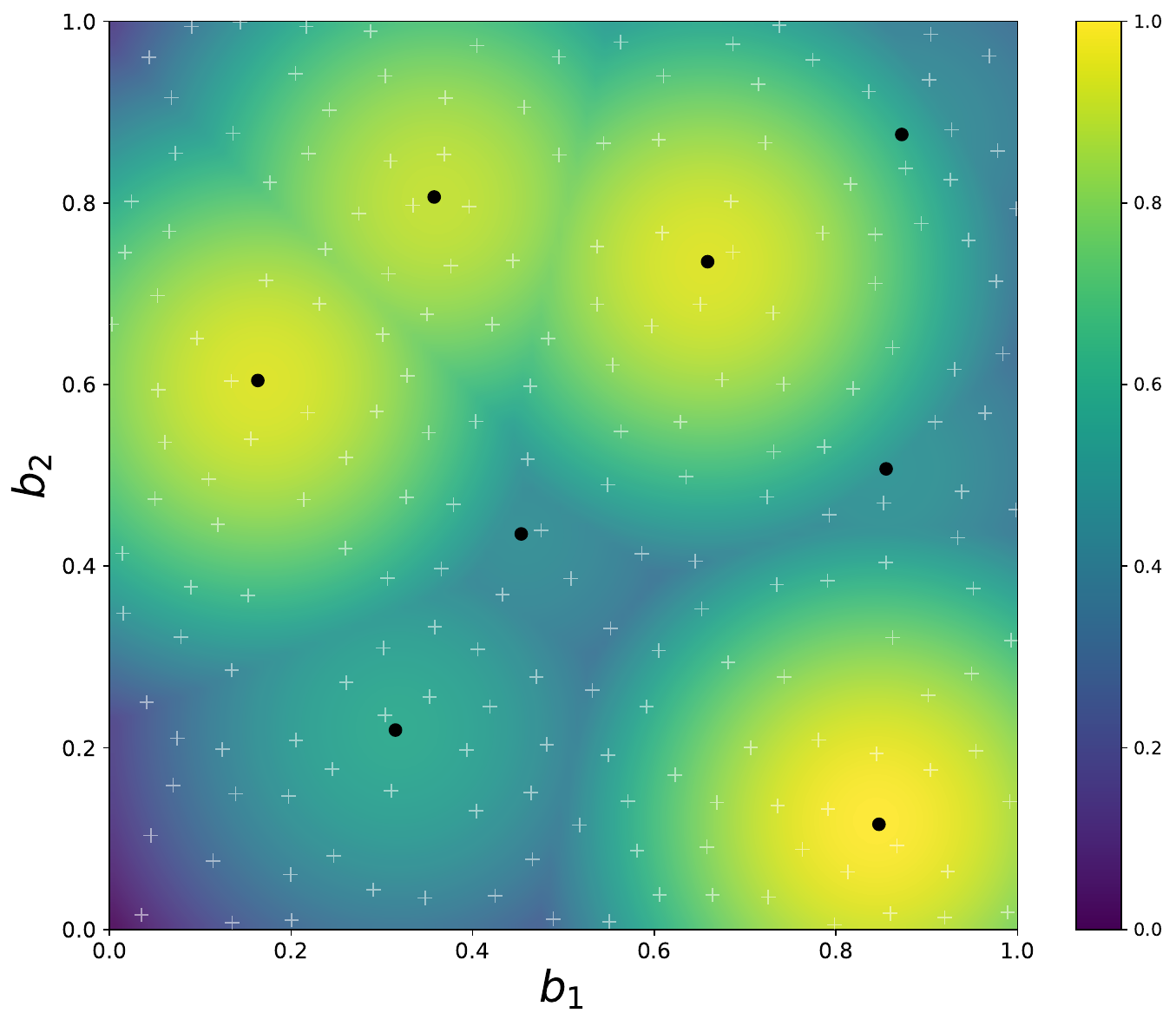}}
\caption{\textbf{Illustration of Different QD Optimization Algorithms.} \textbf{(a) MAP-Elites}~\cite{mouret2015illuminating} discretizes the behavior space into a fixed grid with some cells ($25$ in this case) and finds one local optimal solution per cell. \textbf{(b) CVT-MAP-Elites}~\cite{vassiliades2017using} partitions the behavior space into a user-defined number of convex regions ($8$ in this case) and finds one local optimal solution per region. \textbf{(c) Soft QD}~\cite{hedayatian2025soft} uses several solutions ($8$ black dots in this case) to illuminate the whole behavior space without explicit partition. \textbf{(d) QD as MOO (Ours)} reformulates the QD optimization as an MOO problem with many optimization objectives ($200$ white crosses in this case) spanning the whole behavior space, and finds a smaller set of diverse solutions ($8$ black dots in this case) to tackle all optimization objectives in a collaborative manner.}
\label{fig_motivation}
\end{figure*}

Instead of proposing a new QD algorithm, this work reformulates the QD optimization, particularly the soft QD problem~\cite{hedayatian2025soft}, as a multi-objective optimization (MOO) problem~\cite{miettinen1999nonlinear,ehrgott2005multicriteria, ehrgott2025fifty}. This novel perspective reveals an intrinsic, yet previously unexplored, connection between the two fields. While both aim to obtain a set of diverse solutions, QD explicitly seeks high-performing coverage of a behavior space, whereas MOO focuses on balancing trade-offs among multiple competing objectives. By formally establishing this linkage, we demonstrate how MOO principles and algorithms can be directly leveraged to solve QD problems, offering practical efficiency alongside theoretical guarantees. Consequently, this reformulation not only provides a new lens for developing and analyzing QD algorithms but also opens new research directions by bridging these two well-established fields. 

The contributions of this work are summarized as follows.~\footnote{Our code will be fully open-sourced upon publication.}
\begin{itemize}
    
    \item We reformulate the QD optimization as an MOO problem whereby behavior space coverage is achieved by optimizing a set of uniformly distributed objectives.

    \item We leverage set-based scalarization methods from MOO to solve the QD problem, which facilitates a collaborative search for a diverse set of high-performing solutions with favorable theoretical guarantees.
    
    \item We conduct experiments across multiple QD benchmarks to demonstrate our method's competitive performance, which validates the effectiveness of our proposed reformulation. 
    
\end{itemize}

\section{Preliminaries and Related Work for Quality-Diversity Optimization}

\paragraph{Problem Definition}

For a QD optimization problem, a solution $\vx \in \mathcal{X}$ produces two outputs upon evaluation:
\begin{equation}
\big( f(\vx), \ \vb(\vx) \big) \leftarrow \vF(\vx),
\label{eq_qd_definition}
\end{equation}
where $f(\vx) \in \mathbb{R}$ is a scalar performance measure and $\vb(\vx) \in \mathcal{B} \subset \mathbb{R}^d$ is a $d$-dimensional behavior descriptor in the behavior space $\mathcal{B}$.  
Formally, QD aims to produce a set of high-performing solutions that collectively cover the behavior space. For each potential descriptor $\vb^{(k)} \in \mathcal{B}$, the corresponding optimal solution is defined as:
\begin{align}
\vx^{*(k)} = \arg \max_{\vx \in \mathcal{X}} f(\vx) \quad 
\text{s.t.} \quad \vb(\vx) = \vb^{(k)}.
\label{eq_qd_problem}
\end{align}
The set $\{ \vx^{*(k)} \}_{k=1}^{K}$ is the ideal QD solution set and satisfies two core properties:
\begin{enumerate}
      \item \textbf{Diversity:} The corresponding behavior descriptors $\{\vb(\vx^{*(k)})\}_{k=1}^{K}$ span a broad region of $\mathcal{B}$.
      
    \item \textbf{Quality:} Each $\vx^{*(k)}$ achieves a high objective value $f(\vx^{*(k)})$ with its associated behavior $\vb(\vx^{*(k)}) = \vb^{(k)}$.
    
\end{enumerate}

Since $\mathcal{B}$ is typically continuous, the QD problem~(\ref{eq_qd_problem}) involves infinitely many subproblems. In practice, QD algorithms therefore seek a finite yet representative set of solutions that approximate the ideal coverage of $\mathcal{B}$. It is important to note that diversity is defined in the behavior space $\mathcal{B}$, not in the solution space $\mathcal{X}$. Therefore, generating a set of solutions that are diverse in $\mathcal{X}$ does not guarantee a diverse coverage of $\mathcal{B}$. This inherent discrepancy necessitates the need to develop specialized QD algorithms that directly target diversity in the behavior space.

\paragraph{Grid-based Methods}

A widely adopted strategy in QD optimization is to discretize the continuous behavior space $\mathcal{B}$ into a finite set of $M$ non-overlapping cells $\{c_m\}_{m=1}^{M}$, thereby forming a structured archive $\mathcal{A}$. This discretization effectively reduces the original problem to a collection of $M$ cell-specific optimization subproblems:
\begin{equation}
    \forall \ c_m \in \mathcal{A}, \quad \max_{\vx \in \mathcal{X}} f(\vx) \quad \text{s.t.} \quad \vb(\vx) \in c_m.
    \label{eq_grid_subproblem}
\end{equation}
The quality of a solution set $\{ \vx^{(k)} \}_{k=1}^{K}$ is commonly evaluated by the QD Score~\cite{pugh2016quality}, which aggregates the performance of the best solution within each cell:
\begin{equation}
    \text{QD-Score} = \sum_{c_m \in \mathcal{A}} \max_{\vx \in \{ \vx^{(k)} \}_{k=1}^{K}: \vb(\vx) \in c_m} f(\vx).
    \label{eq_qd_score}
\end{equation}
This metric inherently balances coverage (number of occupied cells) and the quality of solutions within them.

Within this discretization paradigm, \textbf{MAP-Elites}~\cite{mouret2015illuminating} is a seminal and foundational algorithm. It initializes a fixed-grid archive over $\mathcal{B}$ and employs a stochastic population-based algorithm (e.g., evolutionary algorithm) to iteratively discover and retain the highest-performing candidate for each cell. The performance of each cell is solely represented by the best performing solution as shown in Figure~\ref{fig_motivation}(a). Despite its conceptual simplicity, MAP-Elites has proven remarkably powerful and has inspired extensive subsequent research. A fundamental limitation of such fixed-grid methods is the exponential growth in the number of cells $M$ as the dimensionality $d$ of the behavior space $\mathcal{B}$ increases, a classic instance of the \emph{curse of dimensionality}. Consequently, maintaining a fine-grained grid becomes computationally infeasible even with moderate dimensions. 

To mitigate this scalability issue, more advanced discretization schemes have been developed. \textbf{CVT-MAP-Elites}~\cite{vassiliades2017using} replaces the uniform grid with a centroidal Voronoi tessellation (CVT) defined by a fixed number of $M$ centroids. Each cell $c_m$ contains all points in $\mathcal{B}$ closer to its assigned centroid $\vs_m$ than to any other, i.e., $c_m = \{ \vb \in \mathcal{B} \mid  \arg\min_{i} \|\vb - \vs_i\| = m \}$. This construction tends to produce cells of comparable volume and avoids exponential scaling in $d$ as shown in Figure~\ref{fig_motivation}(b). Other strategies include hierarchical behavioral repertoire~\cite{cully2018hierarchical}, dimensionality reduction~\cite{grillotti2022unsupervised}, and fast centroid generation~\cite{mouret2023fast}. Nevertheless, effectively searching within large, high-dimensional cells remains a persistent challenge for grid-based methods.

\paragraph{Novelty Search} 

Novelty Search~\cite{lehman2008exploiting, lehman2011abandoning} offers an alternative QD paradigm that avoids discretizing the behavior space $\mathcal{B}$. Its core insight is that directly seeking behavioral novelty, measured by the distance between behavior descriptors, can sometimes lead to better performance than directly optimizing the objective. \textbf{NSLC}~\cite{lehman2011evolving} extends this idea by introducing local quality competitions among behaviorally similar solutions, thereby jointly improving diversity and performance. A recent extension, Dominated Novelty Search (DNS)~\cite{bahlous2025dominated} implements local competition via dynamic fitness transformation, eliminating the need for predefined bounds or fixed distance thresholds in novelty search.

The term Quality-Diversity Optimization was formally introduced by \citet{pugh2015confronting, pugh2016quality}, accompanied by a systematic analysis of existing paradigms. \citet{cully2017quality} propose a unifying modular framework that encompasses both grid-based and novelty search approaches. A significant subsequent research direction focuses on enhancing search efficiency by integrating advanced evolutionary strategies~\cite{fontaine2020covariance, tjanaka2023training,batra2024proximal}. Many gradient-aware QD algorithms have also been developed~\cite{nilsson2021policy,fontaine2021differentiable,fontaine2023covariance}, which leverage gradient information for achieve faster and more sample-efficient search.

\paragraph{Continuous Quality-Diversity Optimization}

The QD methods with hard-assignment may fail to capture the continuous nature of the behavior space $\mathcal{B}$ and can be restrictive in practice \cite{kent2022discretization}. Surrogate model-based techniques have been explored to enable continuous search for expensive black-box QD problems \cite{cully2015robots, gaier2017data, kent2020bop}, and a continuous QD score has been proposed as a discretization-free evaluation metric \cite{kent2022discretization}.

Recently, \citet{hedayatian2025soft} introduced \textbf{Soft QD}, a kernel-based method for continuous coverage of $\mathcal{B}$. Here, each solution $\vx^{(k)}$ is viewed as illuminating a region around its behavior descriptor $\vb(\vx^{(k)})$ with intensity proportional to its quality $f(\vx^{(k)})$ as illustrated in Figure~\ref{fig_motivation}(c). The influence of a solution decays smoothly with the behavioral distance from its behavior descriptor $\vb(\vx^{(k)})$, modeled by a Gaussian kernel. The aggregate illumination at any point $\vb \in \mathcal{B}$ is defined as  the maximum intensity contributed by any solution in the solution set $\{\vx^{(k)}\}_{k=1}^K$:
\begin{equation}
    v(\vb; \{\vx^{(k)}\}_{k=1}^K) 
    = \max_{ k } \left[ f(\vx^{(k)}) \cdot e^{ - \frac{\| \vb - \vb(\vx^{(k)}) \|^2}{2\sigma^2} } \right],
    \label{eq_behavior_value}
\end{equation}
where $\sigma$ is a kernel width parameter controlling the radius of influence. The overall performance of the solution set $\{\vx^{(k)}\}_{k=1}^K$ is quantified by integrating over the entire behavior space $\mathcal{B}$, resulting in the \emph{Soft QD Score}:
\begin{equation}
    S(\{\vx^{(k)}\}_{k=1}^K) = \int_{\mathcal{B}} v(\vb; \{\vx^{(k)}\}_{k=1}^K) \, d\vb.
    \label{eq_soft_qd_score}
\end{equation}
This score directly encodes the QD objective. To achieve a high value, the population must contain high-quality solutions that are well-spread across $\mathcal{B}$, as closely clustered solutions contribute overlapping illumination.

Direct optimization of the integral in the soft QD score~\eqref{eq_soft_qd_score} is intractable. The \textbf{Soft QD Using Approximated Diversity (SQUAD)}~\cite{hedayatian2025soft} derives a tractable lower bound, leading to the following differentiable objective:
\begin{align}
    &\tilde{S}(\{\vx^{(k)}\}_{k=1}^K) = \sum_{k=1}^{K} f(\vx^{(k)}) -  \nonumber \\
    & \sum_{1 \leq i < j \leq K} \sqrt{f(\vx^{(i)}) f(\vx^{(j)})} \cdot e^{ - \frac{\| \vb(\vx^{(i)}) - \vb(\vx^{(j)}) \|^2}{\gamma^2} },
    \label{eq_squad_objective}
\end{align}
where $\gamma$ is a repulsion bandwidth, the first term attracts solutions toward high quality, while the second acts as a pairwise repulsion that penalizes behavioral similarity, weighted by the geometric mean of qualities. Assuming both $f(\vx)$ and $\vb(\vx)$ are differentiable \cite{fontaine2021differentiable}, the objective function~\eqref{eq_squad_objective} becomes fully differentiable with respect to $\vx$, which enables efficient gradient-based optimization of the entire solution set without any discrete archive.

\section{Quality-Diversity Optimization as Multi-Objective Optimization}

\subsection{Multi-Objective Optimization Reformulation}

This work reformulates the QD optimization as a particular MOO problem. We adopt the differentiable QD formulation of Soft QD~\cite{hedayatian2025soft} to enable efficient gradient-based optimization. Our core reformulation is to directly optimize a large number of objectives with corresponding behaviors $\vb_m$ spanning the entire behavior space $\mathcal{B}$.

Specifically, let $\{\vb_m\}_{m=1}^{M}$ be a large set of target behavior descriptors densely sampled from the continuous behavior space $\mathcal{B}$ (e.g., $M = 10,000$). For each $\vb_m$, we define an objective function for solution $\vx$ based on its quality $f(\vx)$ and the distance from its behavior descriptor $\vb(\vx)$ to $\vb_m$:
\begin{equation}
    \tilde{v}_m(\vx) = - f(\vx) \cdot e^{ - \frac{\| \vb_m - \vb(\vx) \|^2}{\gamma^2} }.
    \label{eq_mop_objective}
\end{equation}
The negative sign is introduced to align with the standard MOO convention of \emph{minimization}. Thus, a lower value of $\tilde{v}_m(\vx)$ is desirable, indicating either a higher quality $f(\vx)$ or a behavior $\vb(\vx)$ closer to the target $\vb_m$ (or both).

Then, we can define an MOO problem with $M$ objectives:
\begin{equation}
    \min_{\vx \in \mathcal{X}} \tilde{\vv}(\vx) = \Big( \tilde{v}_1(\vx),  \dots, \tilde{v}_M(\vx) \Big).
    \label{eq_mop}
\end{equation}
As $M$ becomes large and ${\vb_m}$ densely covers $\mathcal{B}$, this formulation offers a finite approximation to the original continuous QD problem. Crucially, in non-trivial settings, no single solution can minimize all $M$ objectives simultaneously, since improving performance toward one target $\vb_m$ typically degrades it for others. Therefore, we have the definition of dominance and Pareto optimality for MOO~\cite{miettinen1999nonlinear,ehrgott2005multicriteria}:

\begin{definition}[Dominance]
For $\vx^{(a)}, \vx^{(b)} \in \mathcal{X}$,  $\vx^{(a)}$ is said to \emph{dominate} $\vx^{(b)}$, denoted as $\vx^{(a)} \prec \vx^{(b)}$, if and only if $\tilde{v}_m(\vx^{(a)}) \leq \tilde{v}_m(\vx^{(b)}) \quad \forall m \in \{1,\dots,M\}$ and $\tilde{v}_n(\vx^{(a)}) < \tilde{v}_n(\vx^{(b)}) \quad \exists n \in \{1,\dots,M\}$. Furthermore, $\vx^{(a)}$ \emph{strictly dominates} $\vx^{(b)}$, denoted as $\vx^{(a)}\prec_{\text{strict}} \vx^{(b)}$, if $\tilde{v}_m(\vx^{(a)}) < \tilde{v}_m(\vx^{(b)}) \quad \forall m \in \{1,\dots,M\}$.
\end{definition}

\begin{definition}[(Weakly) Pareto Optimal Solution]
A solution $\vx^{*} \in \mathcal{X}$ is \emph{Pareto optimal} if no other $\vx \in \mathcal{X}$ dominates it, i.e., $\nexists  \vx \in \mathcal{X}$ such that $\vx \prec \vx^{*}$. It is weakly Pareto optimal if no other solution strictly dominates it.
\label{def_pareto_optimality}
\end{definition}

\begin{definition}[Pareto Set and Pareto Front]
The set of all Pareto optimal solutions is called the \emph{Pareto set}:
\begin{equation}
    \mathcal{X}^* = \{ \vx \in \mathcal{X} \mid \nexists \, \hat{\vx} \in \mathcal{X} \text{ such that } \hat{\vx} \prec \vx \}.
\end{equation}
Its image in the objective space is called the \emph{Pareto front}:
\begin{equation}
    \tilde{\vv}(\mathcal{X}^*) = \{ \tilde{\vv}(\vx) \in \mathbb{R}^M \mid \vx \in \mathcal{X}^* \}.
\end{equation}
\end{definition}

A single Pareto optimal solution only captures one specific trade-off among the $M$ objectives and typically does not minimize all objectives simultaneously. Conversely, finding $M$ specialized solutions, each dedicated to a single objective $\tilde{v}_m(\vx)$, is computationally infeasible for large $M$ (e.g., $10,000$). We therefore adopt a more practical alternative recently developed for MOO with a vast number of objectives~\cite{ding2024efficient, li2025many, lin2025few, liu2024many, liu2025few, maus2025multi}. The goal here is to find a \emph{small} set of $K$ solutions (where $K \ll M$) that collaboratively and complementarily address all $M$ objectives. Specifically, we aim to ensure that each objective $\tilde{v}_m(\vx)$ is well optimized by at least one solution in the set. 

Formally, let $\vX_K = \{\vx^{(k)}\}_{k=1}^{K} \subset \mathcal{X}$ be a candidate solution set. Its quality is evaluated by taking, for each objective $\tilde{v}_m$, the best (minimum) value attained within the set. This leads to the set-based multi-objective optimization problem:
\begin{align}
    \min_{\vX_K \subset \mathcal{X}}& \; \tilde{\vv}(\vX_K) = \nonumber \\ 
    &\left( \min_{\vx \in \vX_K} \tilde{v}_1(\vx),  \; \dots, \; \min_{\vx \in \vX_K} \tilde{v}_M(\vx) \right).
    \label{eq_set_mop}
\end{align}
When $K=1$, this problem reduces to the standard MOO problem (\ref{eq_mop}), where a single solution must balance all $M$ objectives. When $K \geq M$, it becomes degenerate, as each objective can be assigned a dedicated solution, effectively decoupling the optimization into $M$ independent single-objective problems. Any additional solution is redundant.

This work focuses on the practical and computationally meaningful regime where $1 < K \ll M$. In this setting, $K$ solutions must collaborate to cover the diverse objectives, directly mirroring the core QD objective of covering the behavior space $\mathcal{B}$ with a set of diverse high-quality solutions.

\subsection{Set Scalarization for QD Optimization}

The formulation in (\ref{eq_set_mop}) remains a multi-objective optimization problem. When $K \ll M$, it is generally impossible for a small set $\vX_K$ to simultaneously optimize all $M$ objectives. Recently, some methods have been developed to tackle such set-based multi-objective problems. For example, MosT~\cite{li2025many} combines optimal transport with the Multiple Gradient Descent Algorithm (MGDA)~\cite{sener2018multi} to find a solution set covering diverse regions of the Pareto front. However, MGDA incurs high computational overhead in high-dimensional objective spaces, making it impractical for the QD setting, where $M$ can be very large. To ensure scalability, we turn to decomposition-based multi-objective optimization methods~\cite{zhang2007moea}, which scalarize the multi-objective optimization problem into a single-objective one. Recently, a few gradient-based set scalarization approaches~\cite{ding2024efficient,lin2025few} have been proposed to tackle the set optimization problem~(\ref{eq_set_mop}), which we adapt and extend for the QD optimization in this work.

\paragraph{Sum-of-Minimum (SoM) Scalarization}

A straightforward scalarization approach for the set-based MOO problem is the Sum-of-Minimum (SoM) method~\cite{ding2024efficient}. It aggregates performance by summing the best value achieved for each objective across the solution set:
\begin{equation}
    g^{\text{SoM}}(\vX_K) = \sum_{m=1}^{M} \; \min_{\vx \in \vX_K} \tilde{v}_m(\vx).
    \label{eq_som}
\end{equation}
Minimizing $g^{\text{SoM}}(\vX_K)$ encourages the solution set $\vX_K$ to collectively produce favorable values across all objectives. This scalarization approach has direct conceptual connections to existing QD metrics and algorithms:

\begin{itemize}
    
    \item \textbf{QD Score:} SoM can be viewed as a soft variant of the QD score~(\ref{eq_qd_score}), but with $K \ll M$ solutions that are not restricted to discrete cells.

    \item  \textbf{CVT:} While CVT~\cite{vassiliades2017using} partitions the behavior space $\mathcal{B}$ into cells based on behavior clustering, SoM implicitly groups objectives based on which solution in $\vX_K$ achieves the best value for each.

    \item \textbf{Soft QD Score:} SoM constitutes a discretized approximation of the Soft QD score~\eqref{eq_soft_qd_score}.
    
    \item \textbf{SQUAD:} The SQUAD objective (\ref{eq_squad_objective}) is a lower bound to the Soft QD score, employing local pairwise repulsion for diversity. In contrast, SoM encourages diversity through global attraction to different objectives.

\end{itemize}

When the decision-maker has explicit preferences over the $M$ objectives, a weighted variant can be employed $g^{\text{SoM}}(\vX_K; \vlambda) = \sum_{m=1}^{M} \lambda_m  \min_{\vx \in \vX_K} \tilde{v}_m(\vx)$ where $\vlambda = (\lambda_1, \dots, \lambda_M) \in \mathbb{R}^M$ represents the relative importance of each objective $\tilde{v}_m(\vx)$ (and thus each target behavior $\vb_m$). In this work, we consider the standard QD setting where all behaviors are equally important, and hence simply set $\lambda_m = 1$ for all $M$ objectives.

\paragraph{Tchebycheff Set (TCH-Set) Scalarization}

An alternative approach is the Tchebycheff Set (TCH-Set) scalarization~\cite{lin2025few}: 
\begin{equation}
    g^{\text{TCH-Set}}(\vX_K) = \max_{1 \le m \le M} \; \min_{\vx \in \vX_K} \left( \tilde{v}_m(\vx) - z_m^* \right),
    \label{eq_tch_set_scalarization}
\end{equation}
where $\vz^* = (z_1^*, \dots, z_M^*)$ is a reference point. Ideally, $z_m^*$ represents the optimal value for the objective $m$, i.e., $z_m^* = \min_{\vx \in \mathcal{X}} \tilde{v}_m(\vx)$. In practice, when the ideal values are unknown in advance, they are typically approximated during optimization, for instance, by setting $\hat{z}_m = \min_{\vx \in \vX_K} \tilde{v}_m(\vx) - \varepsilon$ with a small positive $\varepsilon > 0$. 

Minimizing $g^{\text{TCH-Set}}(\vX_K)$ focuses effort on improving the objective currently performing worst. This min-max formulation ensures that no objective is severely neglected, promoting a balanced improvement across all objectives. Consequently, it encourages the solution set $\vX_K$ to achieve comprehensive coverage of the QD behavior space $\mathcal{B}$ by preventing over-specialization on a subset of target behaviors.

\begin{figure*}[t]
    \centering
    \includegraphics[width= 1 \linewidth]{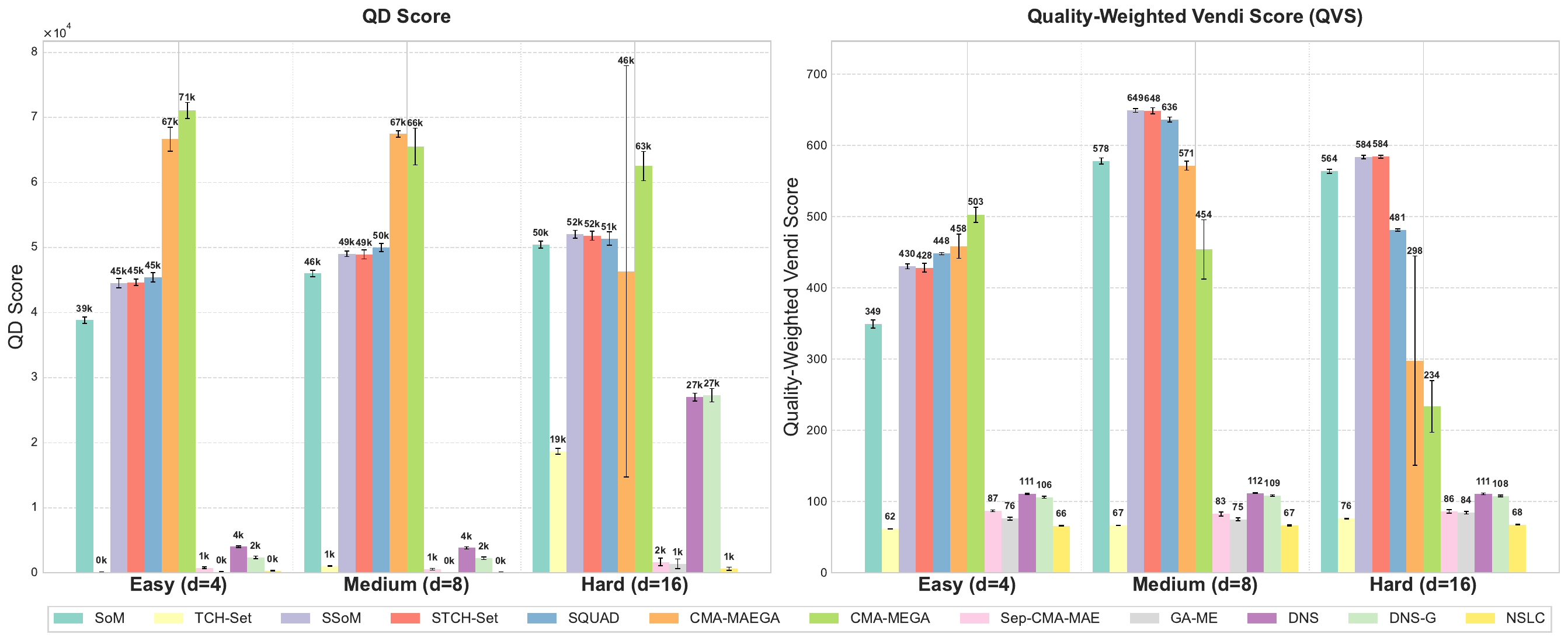}
    \caption{ \textbf{Performance on Linear Projection (LP) across $4$, $8$ and $16$-dimensional behavior space.} Reported values are the mean and standard deviation over 10 independent runs. The smooth MOO-based methods, SSoM and STCH-Set, achieve the best QVS on both the Medium ($d=8$) and Hard ($d=16$) tasks. They maintain stable performance across all other comparisons, on par with SQUAD.}
    \label{fig_lp_results}
\end{figure*}

\begin{table*}[ht]
\centering
\setlength{\tabcolsep}{2pt}
\caption{\textbf{Performance on the Image Composition (IC) Benchmark.} Reported values are the mean and standard deviation over 10 independent runs. For each metric, the best result is highlighted in \textbf{bold} with a gray background, and the second best is \underline{underlined}.}
\label{table_image_composition}
\begin{tabular}{lccccccc}
\toprule
& \multicolumn{2}{c}{\textbf{Quality}} & \multicolumn{2}{c}{\textbf{Diversity}} & \multicolumn{2}{c}{\textbf{Overall Performance}} \\
\cmidrule(lr){2-3} \cmidrule(lr){4-5} \cmidrule(lr){6-7}
\textbf{Algorithm} & \textbf{Mean Objective} & \textbf{Max Objective} & \textbf{Coverage} & \textbf{Vendi Score} & \textbf{QD Score} & \textbf{QVS} \\
\midrule
\multicolumn{7}{c}{\textbf{MOO Methods}} \\
SoM                & $65.12 \pm 0.33$        & $93.81 \pm 0.27$       & $5.92 \pm 2.14$   & $1.67 \pm 0.08$      & $5307.5 \pm 120.5$ & $108.53 \pm 5.86$ \\
TCH-Set            & $60.30 \pm 0.02$        & $64.15 \pm 0.48$       & $0.10 \pm 0.00$   & $1.00 \pm 0.00$      & $64.15 \pm 0.48$ & $60.47 \pm 0.02$ \\
SSoM               & $76.41 \pm 0.15$        & $\underline{93.89 \pm 0.21}$       & $\underline{6.25 \pm 0.15}$   & $4.17 \pm 0.05$      & $\underline{5627.7 \pm 127.4}$ & $318.83 \pm 4.15$ \\
STCH-Set           & $\underline{76.41 \pm 0.13}$        & $\cellcolor{gray!20}\mathbf{94.12 \pm 0.35}$ & $\cellcolor{gray!20}\mathbf{6.31 \pm 0.48}$ & $\underline{4.19 \pm 0.05}$      & $\cellcolor{gray!20}\mathbf{5677.7 \pm 141.9}$ & $\underline{319.99 \pm 3.82}$ \\
\addlinespace
\cmidrule(r){1-7}
\multicolumn{7}{c}{\textbf{QD Methods}} \\
SQUAD              & $\cellcolor{gray!20}\mathbf{83.37 \pm 0.06}$ & $93.58 \pm 0.32$       & $5.68 \pm 0.20$ & $\cellcolor{gray!20}\mathbf{5.49 \pm 0.01}$ & $5086.2 \pm 173.0$ & $\cellcolor{gray!20}\mathbf{457.35 \pm 0.71}$ \\
CMA-MAEGA          & $74.95 \pm 0.32$        & $90.61 \pm 2.07$       & $5.81 \pm 0.25$   & $3.94 \pm 0.07$      & $4588.9 \pm 193.9$ & $295.17 \pm 5.19$ \\
CMA-MEGA           & $76.16 \pm 0.81$        & $85.39 \pm 3.36$       & $4.50 \pm 1.58$   & $3.24 \pm 0.76$      & $3539.8 \pm 1241.9$ & $246.70 \pm 56.81$ \\
Sep-CMA-MAE        & $72.14 \pm 0.11$        & $74.57 \pm 0.11$       & $0.46 \pm 0.09$   & $1.31 \pm 0.04$      & $341.5 \pm 67.6$ & $94.71 \pm 2.61$ \\
GA-ME              & $73.09 \pm 1.17$        & $74.30 \pm 1.06$       & $0.22 \pm 0.12$   & $1.18 \pm 0.15$      & $168.9 \pm 90.7$ & $86.16 \pm 10.67$ \\
DNS                & $71.22 \pm 0.39$        & $74.53 \pm 0.89$       & $1.54 \pm 0.08$   & $1.62 \pm 0.02$      & $1143.7 \pm 56.6$ & $115.06 \pm 1.64$ \\
DNS-G              & $74.47 \pm 0.13$        & $76.55 \pm 0.47$       & $1.51 \pm 0.12$   & $1.67 \pm 0.01$      & $1163.3 \pm 91.5$ & $124.64 \pm 0.95$ \\
NSLC               & $72.21 \pm 0.50$        & $74.07 \pm 0.25$       & $0.56 \pm 0.10$   & $1.21 \pm 0.04$      & $417.5 \pm 76.4$ & $87.15 \pm 2.48$ \\
\bottomrule
\end{tabular}
\end{table*}

\begin{table*}[ht]
\centering
\small
\setlength{\tabcolsep}{1.5 pt}
\caption{\textbf{Performance on the Latent Space Illumination (LSI) Benchmark}. The QD Score is shown in ($\times 10^3$). Algorithms with a negative mean objective have a corresponding zero QVS score $0.0 \pm 0.0$*. The results of QD methods are directly from \cite{hedayatian2025soft}. Each MOO method is independently run $5$ times. For each metric, the best result is highlighted in \textbf{bold} with a gray background, and the second best is \underline{underlined}.}
\label{table_latent_space_illumination}
\begin{tabular}{lcccccccc}
\toprule
& \multicolumn{4}{c}{LSI} & \multicolumn{4}{c}{LSI (Hard)} \\
\cmidrule(lr){2-5} \cmidrule(lr){6-9}
\textbf{Algorithm} & 
\textbf{Coverage} & 
\textbf{Vendi Score} & 
\textbf{QD Score } & 
\textbf{QVS} & 
\textbf{Coverage} & 
\textbf{Vendi Score} & 
\textbf{QD Score} & 
\textbf{QVS} \\
\midrule
\multicolumn{9}{c}{\textbf{MOO Methods}} \\
\text{SoM}              & $2.77 \pm 0.11$ & $1.13 \pm 0.00$ & $-10.67 \pm 2.80$ & $0.0 \pm 0.0$* & $0.52 \pm 0.15$ & $1.05 \pm 0.00$ & $-1.28 \pm 0.91$ & $0.0 \pm 0.0$* \\
\text{TCH-Set}          & $2.60 \pm 0.15$ & $1.13 \pm 0.01$ & $-10.46 \pm 2.76$ & $0.0 \pm 0.0$* & $0.49 \pm 0.17$ & $1.05 \pm 0.00$ & $-1.22 \pm 1.04$ & $0.0 \pm 0.0$* \\
\text{SSoM}             & $12.50 \pm 0.74$ & $\underline{1.93 \pm 0.06}$ & $6.26 \pm 0.86$ & $56.93 \pm 26.11$ & $3.71 \pm 0.10$ & $\underline{1.61 \pm 0.01}$ & $\cellcolor{gray!25}\mathbf{3.08 \pm 0.08}$ & $18.66 \pm 10.03$ \\
\text{STCH-Set}         & $11.52 \pm 0.59$ & $\underline{1.93 \pm 0.03}$ & $7.38 \pm 0.76$ & $99.73 \pm 3.05$ & $3.55 \pm 0.30$ & $\underline{1.61 \pm 0.00}$ & $\underline{2.95 \pm 0.24}$ & $17.23 \pm 4.61$ \\
\addlinespace
\cmidrule(r){1-9}
\multicolumn{9}{c}{\textbf{QD Methods}} \\
SQUAD       & $\cellcolor{gray!25}\mathbf{32.4 \pm 0.5}$ & $\cellcolor{gray!25}\mathbf{2.22 \pm 0.03}$ & $\cellcolor{gray!25}\mathbf{13.41 \pm 0.19}$ & $\cellcolor{gray!25}\mathbf{177.0 \pm 2.8}$   & $6.0 \pm 0.2$   & $\cellcolor{gray!25}\mathbf{1.83 \pm 0.00}$ & $2.55 \pm 0.08$ & $\cellcolor{gray!25}\mathbf{151.3 \pm 0.1}$   \\
CMA-MAEGA   & $15.9 \pm 2.0$ & $1.57 \pm 0.07$ & $6.82 \pm 0.10$  & $121.6 \pm 9.7$   & $0.9 \pm 0.2$   & $1.22 \pm 0.02$ & $0.39 \pm 0.07$  & $\underline{99.3 \pm 1.0}$    \\
CMA-MEGA    & $\underline{19.7 \pm 0.2}$ & $1.68 \pm 0.01$ & $\underline{8.70 \pm 0.07}$  & $\underline{140.1 \pm 1.7}$   & $0.6 \pm 0.1$   & $1.10 \pm 0.02$ & $0.27 \pm 0.05$  & $92.8 \pm 1.5$    \\
Sep-CMA-MAE & $0.4 \pm 0.1$  & $1.01 \pm 0.01$ & $-0.59 \pm 0.45$ & $0.0 \pm 0.0$*    & $0.2 \pm 0.0$   & $1.00 \pm 0.00$ & $0.02 \pm 0.04$   & $0.0 \pm 0.0$*    \\
GA-ME       & $6.8 \pm 0.3$  & $1.25 \pm 0.01$ & $-14.90 \pm 1.67$& $0.0 \pm 0.0$*    & $0.2 \pm 0.0$   & $1.04 \pm 0.01$ & $0.08 \pm 0.00$   & $0.0 \pm 0.0$*    \\
DNS         & $10.2 \pm 0.3$ & $1.39 \pm 0.00$ & $-9.31 \pm 2.52$ & $0.0 \pm 0.0$*    & $\cellcolor{gray!25}\mathbf{10.2 \pm 0.2}$  & $1.38 \pm 0.00$ & $-11.27 \pm 1.46$ & $0.0 \pm 0.0$*    \\
DNS-G       & $9.3 \pm 0.0$  & $1.36 \pm 0.00$ & $-8.53 \pm 1.50$ & $0.0 \pm 0.0$*    & $\underline{9.1 \pm 0.2}$   & $1.35 \pm 0.01$ & $-6.81 \pm 0.47$  & $0.0 \pm 0.0$*    \\
\bottomrule
\end{tabular}
\vspace{1.5em}
\end{table*}

\paragraph{Need for Smooth Approximations}

Both the SoM~\eqref{eq_som} and TCH-Set~\eqref{eq_tch_set_scalarization} scalarization involve non-differentiable $\min$ and $\max$ operators. This non-smoothness poses a significant challenge for gradient-based optimization, especially in the QD setting where both the number of objectives $M$ and the solution set size $K$ can be large.

To enable efficient optimization while retaining the essential properties of these scalarizations, following the approach in \citet{lin2025few}, we employ smooth approximations of the $\min$ and $\max$ operators based on the log-sum-exp function~\cite{beck2012smoothing}. Specifically, for a set of values $\{a_i\}_{i=1}^n$, the smooth maximum and smooth minimum operators are defined as:
\begin{align}
    \smax_{\mu}(a_1,\dots,a_n) &= \mu \log\left( \sum_{i=1}^n e^{a_i / \mu} \right), \\
    \smin_{\mu}(a_1,\dots,a_n) &= -\mu \log\left( \sum_{i=1}^n e^{-a_i / \mu} \right),
\end{align}
where $\mu > 0$ is a smoothing parameter. As $\mu \to 0$, these approximations converge to the $\max$ and $\min$ operators.

\paragraph{Smooth Sum-of-Minimum (SSoM) Scalarization}

By replacing the $\min$ operator in \eqref{eq_som} with $\smin_{\mu}$, we obtain the Smooth Sum-of-Minimum (SSoM) scalarization:
\begin{equation}
    g_{\mu}^{\text{SSoM}}(\vX_K) = -\sum_{m=1}^{M} \mu \log\left( \sum_{k=1}^{K} e^{-\frac{\tilde{v}_m(\vx^{(k)})}{\mu}} \right),
    \label{eq_ssom}
\end{equation}
which provides a fully differentiable objective suitable for gradient-based optimization.

\paragraph{Smooth Tchebycheff Set (STCH-Set) Scalarization}

In a similar way, we apply smoothing to the TCH-Set scalarization (\ref{eq_tch_set_scalarization}), replacing both the inner $\min$ and outer $\max$ operators with their smooth counterparts. Using the same smoothing parameter $\mu$ throughout, the Smooth Tchebycheff Set (STCH-Set) objective simplifies to:
\begin{align}
    &g_{\mu}^{\text{STCH-Set}}(\vX_K) \nonumber \\
    &= \mu \log\left( \sum_{m=1}^{M} e^{  - \log\left( \sum_{k=1}^{K} e^{- \frac{\tilde{v}_m(\vx^{(k)})}{\mu}} \right) - z_m^* } \right).
    \label{eq_stch_set_scalarization_same_u}
\end{align}
Due to space constraints, detailed derivations from TCH-Set~\eqref{eq_tch_set_scalarization} to STCH-Set~\eqref{eq_stch_set_scalarization_same_u} are deferred to the Appendix~\ref{supp_sec_ssom_stch-set}.

\subsection{Theoretical Properties}

Through our reformulation, the set scalarization methods inherit well-established theoretical guarantees from MOO, while providing properties desirable for QD. Many theoretical properties discussed here have been formally established in the MOO literature~\cite{ding2024efficient, liu2024many,lin2024smooth,liu2025few}. In this work, we complete some missing analyses for these set scalarizations and adapt them to the QD context. Detailed proofs and extended discussions are provided in Appendix~\ref{supp_sec_proof}.

First, the set scalarization functions satisfy both monotonicity and supermodularity (for the negative QD objective), which aligns with the promising properties of Soft QD Score~\cite{hedayatian2025soft}:

\begin{theorem}[Monotonicity]
Let $g$ be a function among $\{g^{\text{SoM}}, g_{\mu}^{\text{SSoM}}, g_{\mu}^{\text{STCH-Set}}\}$. For any two solution sets $\vX_U \subseteq \vX_W \subseteq \mathcal{X}$ with $1 \leq U \leq W$, it holds that $g(\vX_U) \geq g(\vX_W)$. Besides, if all reference points are equal ($z_1^* = \cdots = z_M^*$), $g^{\text{TCH-Set}}(\vX_U) \geq g^{\text{TCH-Set}}(\vX_W)$ also holds.
\label{thm_informal_monotonicity}
\end{theorem}

\begin{theorem}[Supermodularity]
Let $g$ be a function among $\{g^{\text{SoM}}, g_{\mu}^{\text{SSoM}}, g_{\mu}^{\text{STCH-Set}}\}$. For any two solution sets $\vX_U \subseteq \vX_W \subseteq \mathcal{X}$ with $1 \leq U \leq W$ and any solution $\boldsymbol{x}' \in \mathcal{X}$\textbackslash $\vX_W$, it holds that $g(\vX_U) - g(\vX_U \cup \left\{\boldsymbol{x}' \right\}) \geq g(\vX_W) - g(\vX_W \cup \left\{\boldsymbol{x}' \right\})$. Besides, if all reference points are equal ($z_1^* = \cdots = z_M^*$) and $\argmax_{1\leq m\leq M}(\min_{\vx \in \vX_U\cup \left\{\boldsymbol{x}' \right\}} \tilde{v}_m(\vx) -z_m^*)=\argmax_{1\leq m\leq M}(\min_{\vx \in \vX_W} \tilde{v}_m(\vx)-z_m^*)$, the same inequality holds for $g^{\text{TCH-Set}}$.
\label{thm_informal_supermodularity}
\end{theorem}

In addition, these set scalarizations also enjoy good properties from the perspective of multi-objective optimization:

\begin{theorem}[Pareto Optimality of Solutions]
All solutions in the optimal solution set $\vX^*_K$ for SoM~(\ref{eq_som}), SSoM~(\ref{eq_ssom}) or STCH-Set~(\ref{eq_stch_set_scalarization_same_u}) scalarization are Pareto optimal for the original multi-objective optimization problem~(\ref{eq_mop}).
\label{thm_pareto_optimality}
\end{theorem}

\begin{theorem}[Existence of Pareto Optimal Solution]
For the TCH-Set scalarization (\ref{eq_tch_set_scalarization}), there exists at least one optimal solution set $\vX^*_K$ whose solutions are all Pareto optimal for (\ref{eq_mop}). Moreover, if the optimal set $\vX^*_K$ is unique, then all its solutions are Pareto optimal.
\label{thm_tch_set_pareto_optimality}
\end{theorem}

\begin{theorem}[Smooth Approximation]
SSoM scalarization (\ref{eq_ssom}) is a uniform smooth approximation of SoM scalarization (\ref{eq_som}). STCH-Set scalarization (\ref{eq_stch_set_scalarization_same_u}) is a uniform smooth approximation of TCH-Set scalarization (\ref{eq_tch_set_scalarization}).
\label{thm_bounded_approximation}
\end{theorem}

\section{Experiments}

\subsection{Experimental Setting} 

\paragraph{Baselines} We compare four MOO-based methods (SoM, TCH-Set, SSoM, and STCH-Set) against eight representative QD algorithms: (1) CMA-MEGA~\citep{fontaine2021differentiable}, (2) CMA-MAEGA~\citep{fontaine2023covariance}, (3) Sep-CMA-MAE~\citep{tjanaka2023training}, (4) Gradient-Assisted MAP-Elites (GA-ME) adapted from PGA-ME~\citep{nilsson2021policy}, (5) Novelty Search with Local Competition (NSLC)~\cite{lehman2011evolving}, (6) Dominated Novelty Search (DNS)~\citep{bahlous2025dominated}, (7) gradient-based DNS (DNS-G), and (8) SQUAD~\cite{hedayatian2025soft}. All implementations are built upon the \texttt{pyribs} library~\cite{tjanaka2023pyribs} and the Soft QD codebase~\cite{hedayatian2025soft}.

\paragraph{Metrics} We evaluate the obtained solution sets using six metrics: Mean Objective and Max Objective for quality; Coverage and Vendi Score~\cite{friedman2023vendi} for diversity; and QD Score~\cite{pugh2016quality} and Quality-weighted Vendi Score (QVS)~\cite{nguyen2024quality} for overall QD performance. All metrics are to be maximized. Detailed experimental settings and further analyses can be found in Appendix~\ref{supp_sec_setting} and~\ref{supp_sec_experiment}.

\paragraph{Goal} Rather than aiming to outperform all existing QD methods, our primary goal is to establish the MOO-based reformulation as a viable and competitive alternative, thereby offering a novel perspective to the field.

\subsection{Main Results}

Following Soft QD~\cite{hedayatian2025soft}, we evaluate the methods on the following three popular differentiable QD benchmarks with details in Appendix~\ref{supp_sec_problem}.

\paragraph{Linear Projection (LP)}

This benchmark is designed to evaluate the scalability of QD algorithms in the behavior spaces~\cite{fontaine2020covariance}. It uses the multimodal Rastrigin function~\cite{rastrigin1974systems,hoffmeister1990genetic} as the optimization objective, and the solution vectors are projected into the behavior space via linear mappings. Experiments are conducted across dimensions $d = \{4, 8, 16 \}$ to assess performance as dimensionality increases. 

The results in Figure~\ref{fig_lp_results} demonstrate the scalability and robustness of the smooth MOO-based set scalarization methods. In the Easy ($d=4$) setting, while gradient-based CMA-MEGA and CMA-MAEGA achieve the highest overall performance, SSoM and STCH-Set deliver competitive QVS on par with SQUAD. As dimensionality increases to Medium ($d=8$), the performance of SSoM and STCH-Set becomes more prominent, and their QVS outperforms that of the best QD method, SQUAD. This trend culminates in the Hard ($d=16$) setting, where SSoM and STCH-Set deliver the highest QVS among all methods, coupled with robust QD Scores. In contrast, several QD methods exhibit substantial performance declines or high variance as dimensionality increases.

\paragraph{Image Composition (IC)}

IC is a differentiable QD benchmark inspired by computational creativity research~\cite{tian2022modern}. In this task, the solution is a set of circles with controllable parameters (position, size, color, and transparency) to reconstruct a target image. The solution quality is measured by the structural similarity between the rendered image and the target, and the diversity is defined using five behavior descriptors, such as color spread and harmony.

The results in Table~\ref{table_image_composition} demonstrate that SSoM and STCH-Set can obtain highly competitive and balanced performance. STCH-Set achieves the best Max Objective, Coverage, and QD Score, while performing the second best on the rest metrics (Mean Objective, Vendi Score and QVS), demonstrating its effective balance between solution quality and broad behavior diversity. SSoM also achieves strong and robust performance across all metrics. In contrast, the Soft QD method SQUAD attains the highest Mean Objective and QVS, indicating a distinct strength in optimizing for high average quality within a diverse set. Among other QD algorithms, CMA-MAEGA and CMA-MEGA show stable results, while several methods struggle to maintain high diversity alongside quality, as reflected in their lower Coverage and QD Scores. These findings highlight that the MOO reformulation, alongside smooth set scalarization, provides a powerful and archive-free alternative for QD optimization.

\paragraph{Latent Space Illumination (LSI)}

The Latent Space Illumination (LSI) benchmark assesses QD algorithms in the StyleGAN2 latent space~\cite{karras2020analyzing} to generate images matching a textual prompt while varying across specified visual attributes~\cite{fontaine2021differentiable,fontaine2023covariance}. Solution quality is measured as the similarity between the generated images and the prompt using CLIP embeddings~\cite {radford2021learning}. Behavior descriptors are defined via contrasting textual pairs to enforce diversity in attributes. LSI includes a base variant with a 2D behavior space and a more challenging 7D variant.

The results in Table~\ref{table_latent_space_illumination} demonstrate that the two smooth MOO-based approaches, SSoM and STCH-Set, consistently achieve a positive overall QD Score and a meaningful QVS like the powerful gradient-based QD methods SQUAD, CMA-MAEGA, and CMA-MEGA. This indicates their ability to successfully navigate the challenging high-dimensional latent space of StyleGAN2 to find solutions that are both high-quality and diverse. Notably, SSoM and STCH-Set achieve the top two QD Scores in the hard LSI setting, outperforming all other QD baselines. They also achieve the second-best Vendi Score in both the base and hard settings. In contrast, the non-smooth MOO methods (SoM, TCH-Set) and several QD algorithms yield negative mean objective values and consequently a zero QVS, failing to produce useful solutions under this metric.

\section{Conclusion, Limitation, and Future Work}

\paragraph{Conclusion} 

This work introduces a novel reformulation of Quality-Diversity (QD) optimization as a multi-objective optimization (MOO) problem. By connecting behavior space coverage to the optimization of a large set of uniformly distributed objectives, we enable the application of efficient set scalarization methods from MOO to tackle QD tasks. These methods aim to find a small set of diverse solutions to collaboratively address many objectives, ensuring both quality and diversity without relying on discrete archives. This approach has favorable theoretical guarantees and supports efficient gradient-based optimization. Extensive experiments across established QD benchmarks confirm the competitive performance of our methods, demonstrating their effectiveness as a flexible and scalable framework for QD. Overall, this reformulation provides a new framework to tackle the QD task from the perspective of MOO, which opens a new pathway for QD algorithm design.

\paragraph{Limitation and Future Work} 

While our reformulation provides a scalable and gradient-efficient framework for QD optimization, it inherently relies on the availability of differentiable objective and behavior descriptor functions. This requirement may restrict its direct application in non-differentiable or black-box settings. In the future, we plan to extend this approach to non-differentiable and reinforcement learning domains, investigate adaptive sampling strategies for objective placement, and explore the integration of preference learning to allow users to steer the search toward desirable regions of the behavior space.

\section*{Impact Statement}




This paper presents work whose goal is to advance the field of Machine Learning. There are many potential societal consequences of our work, none which we feel must be specifically highlighted here.


\bibliographystyle{icml2026}
\bibliography{ref_multiobjective_optimization, ref_multi_task_learning, ref_quality_diversity_optimizaiton}

@string{cvpr = "Proc. CVPR"}

@string{iclr = "Proc. ICLR"}

@string{nips = "Proc. NeurIPS"}

@string{icml = "Proc. ICML"}

@string{cvpr = "{IEEE/CVF} Conference on Computer Vision and Pattern Recognition (CVPR)"}

@string{iclr = "International Conference on Learning Representations (ICLR)"}

@string{nips = "Advances in Neural Information Processing Systems (NeurIPS)"}

@string{icml = "International Conference on Machine Learning (ICML)"}

@string{gecco = "Genetic and Evolutionary Computation Conference (GECCO)"}

@string{ppsn = "International Conference on Parallel Problem Solving from Nature (PPSN)"}

@string{emo = "International Conference on Evolutionary Multi-Criterion Optimization (EMO)"}

@string{ijcai = "International Joint Conferences on Artificial Intelligence (IJCAI)"}

@inproceedings{sener2018multi,
  title={Multi-task learning as multi-objective optimization},
  author={Sener, Ozan and Koltun, Vladlen},
  booktitle={Advances in Neural Information Processing Systems},
  pages={525--536},
  year={2018}
}

@inproceedings{lin2019pareto,
  title={Pareto multi-task learning},
  author={Lin, Xi and Zhen, Hui-Ling and Li, Zhenhua and Zhang, Qingfu and Kwong, Sam},
  booktitle={Advances in Neural Information Processing Systems},
  pages={12060--12070},
  year={2019}
}

@inproceedings{momma2022multi,
  title={A multi-objective/multi-task learning framework induced by pareto stationarity},
  author={Momma, Michinari and Dong, Chaosheng and Liu, Jia},
  booktitle=icml,
  pages={15895--15907},
  year={2022},
  organization={PMLR}
}

@inproceedings{he2024robust,
  title={Robust Multi-Task Learning with Excess Risks},
  author={He, Yifei and Zhou, Shiji and Zhang, Guojun and Yun, Hyokun and Xu, Yi and Zeng, Belinda and Chilimbi, Trishul and Zhao, Han},
  year = 2024,
  booktitle=icml
}

@article{lin2024smooth,
  title={Smooth Tchebycheff Scalarization for Multi-Objective Optimization},
  author={Lin, Xi and Zhang, Xiaoyuan and Yang, Zhiyuan and Liu, Fei and Wang, Zhenkun and Zhang, Qingfu},
  journal={arXiv preprint arXiv:2402.19078},
  year={2024}
}

@article{qiu2024traversing,
  title={Traversing Pareto Optimal Policies: Provably Efficient Multi-Objective Reinforcement Learning},
  author={Qiu, Shuang and Zhang, Dake and Yang, Rui and Lyu, Boxiang and Zhang, Tong},
  journal={arXiv preprint arXiv:2407.17466},
  year={2024}
}

@article{ma2020efficient,
  title={Efficient Continuous Pareto Exploration in Multi-Task Learning},
  author={Ma, Pingchuan and Du, Tao and Matusik, Wojciech},
  journal=icml,
  year={2020}
}

@article{mahapatramulti2020multi,
  title={Multi-Task Learning with User Preferences: Gradient Descent with Controlled Ascent in Pareto Optimization},
  author={Mahapatra, Debabrata and Rajan, Vaibhav},
  journal={Thirty-seventh International Conference on Machine Learning},
  year={2020}
}

@article{lin2020controllable,
      title={Controllable Pareto Multi-Task Learning}, 
      author={Xi Lin and Zhiyuan Yang and Qingfu Zhang and Sam Kwong},
      journal={arXiv preprint arXiv:2010.06313},
      year={2020},
}

@article{navon2021learning,
  title={Learning the Pareto Front with Hypernetworks},
  author={Navon, Aviv and Shamsian, Aviv and Chechik, Gal and Fetaya, Ethan},
  journal=iclr,
  year={2021}
}

@inproceedings{ruchte2021scalable,
  title={Scalable Pareto Front Approximation for Deep Multi-Objective Learning},
  author={Ruchte, Michael and Grabocka, Josif},
  booktitle={IEEE International Conference on Data Mining (ICDM)},
  year={2021}
}

@article{chen2022multi,
  title={Multi-Objective Deep Learning with Adaptive Reference Vectors},
  author={Chen, Weiyu and Kwok, James},
  journal=nips,
  volume={35},
  pages={32723--32735},
  year={2022}
}

@article{dosovitskiy2019you,
  title={You only train once: Loss-conditional training of deep networks},
  author={Dosovitskiy, Alexey and Djolonga, Josip},
  journal=iclr,
  year={2019}
}

@inproceedings{standley2020tasks,
  title={Which tasks should be learned together in multi-task learning?},
  author={Standley, Trevor and Zamir, Amir and Chen, Dawn and Guibas, Leonidas and Malik, Jitendra and Savarese, Silvio},
  booktitle=icml,
  pages={9120--9132},
  year={2020},
  organization={PMLR}
}

@article{fifty2021efficiently,
  title={Efficiently Identifying Task Groupings for Multi-Task Learning},
  author={Fifty, Christopher and Amid, Ehsan and Zhao, Zhe and Yu, Tianhe and Anil, Rohan and Finn, Chelsea},
  journal=nips,
  year={2021}
}

@book{miettinen1999nonlinear,
  title={Nonlinear multiobjective optimization},
  author={Miettinen, Kaisa},
  volume={12},
  year={1999},
  publisher={Springer Science \& Business Media}
}

@book{ehrgott2005multicriteria,
  title={Multicriteria optimization},
  author={Ehrgott, Matthias},
  volume={491},
  year={2005},
  publisher={Springer Science \& Business Media}
}

@book{bertsekas2003convex,
  title={Convex analysis and optimization},
  author={Bertsekas, Dimitri and Nedic, Angelia and Ozdaglar, Asuman},
  volume={1},
  year={2003},
  publisher={Athena Scientific}
}

@book{boyd2004convex,
  title={Convex optimization},
  author={Boyd, Stephen and Vandenberghe, Lieven},
  year={2004},
  publisher={Cambridge university press}
}

@article{ehrgott2025fifty,
  title={Fifty years of multi-objective optimization and decision-making: From mathematical programming to evolutionary computation},
  author={Ehrgott, Matthias and K{\"o}ksalan, Murat and Kadzi{\'n}ski, Mi{\l}osz and Deb, Kalyanmoy},
  journal={European Journal of Operational Research},
  year={2025},
  publisher={Elsevier}
}

@article{deb2002fast,
  title={A fast and elitist multiobjective genetic algorithm: NSGA-II},
  author={Deb, Kalyanmoy and Pratap, Amrit and Agarwal, Sameer and Meyarivan, TAMT},
  journal={IEEE transactions on evolutionary computation},
  volume={6},
  number={2},
  pages={182--197},
  year={2002},
  publisher={IEEE}
}

@inproceedings{zitzler2004indicator,
  title={Indicator-based selection in multiobjective search},
  author={Zitzler, Eckart and K{\"u}nzli, Simon},
  booktitle={International conference on parallel problem solving from nature},
  pages={832--842},
  year={2004},
  organization={Springer}
}

@article{beume2007sms,
  title={SMS-EMOA: Multiobjective selection based on dominated hypervolume},
  author={Beume, Nicola and Naujoks, Boris and Emmerich, Michael},
  journal={European Journal of Operational Research},
  volume={181},
  number={3},
  pages={1653--1669},
  year={2007},
  publisher={Elsevier}
}

@article{bader2011hype,
  title={HypE: An algorithm for fast hypervolume-based many-objective optimization},
  author={Bader, Johannes and Zitzler, Eckart},
  journal={Evolutionary computation},
  volume={19},
  number={1},
  pages={45--76},
  year={2011},
  publisher={MIT Press}
}

@article{das1997a,
	title="A closer look at drawbacks of minimizing weighted sums of objectives for Pareto set generation in multicriteria optimization problems",
	author="Indraneel {Das} and J.E. {Dennis}",
	journal="Structural Optimization",
	volume="14",
	number="1",
	pages="63--69",
	year="1997"
}

@article{zhang2007moea,
  title={{MOEA/D}: A multiobjective evolutionary algorithm based on decomposition},
  author={Zhang, Qingfu and Li, Hui},
  journal={IEEE Transactions on evolutionary computation},
  volume={11},
  number={6},
  pages={712--731},
  year={2007},
  publisher={IEEE}
}

@inproceedings{fleming2005many,
  title={Many-objective optimization: An engineering design perspective},
  author={Fleming, Peter J and Purshouse, Robin C and Lygoe, Robert J},
  booktitle={International conference on evolutionary multi-criterion optimization},
  pages={14--32},
  year={2005},
  organization={Springer}
}

@article{li2015many,
  title={Many-objective evolutionary algorithms: A survey},
  author={Li, Bingdong and Li, Jinlong and Tang, Ke and Yao, Xin},
  journal={ACM Computing Surveys (CSUR)},
  volume={48},
  number={1},
  pages={1--35},
  year={2015},
  publisher={Acm New York, NY, USA}
}

@article{purshouse2007evolutionary,
  title={On the evolutionary optimization of many conflicting objectives},
  author={Purshouse, Robin C and Fleming, Peter J},
  journal={IEEE transactions on evolutionary computation},
  volume={11},
  number={6},
  pages={770--784},
  year={2007},
  publisher={IEEE}
}

@inproceedings{knowles2007quantifying,
  title={Quantifying the effects of objective space dimension in evolutionary multiobjective optimization},
  author={Knowles, Joshua and Corne, David},
  booktitle={Evolutionary Multi-Criterion Optimization: 4th International Conference, EMO 2007, Matsushima, Japan, March 5-8, 2007. Proceedings 4},
  pages={757--771},
  year={2007},
  organization={Springer}
}

@article{sato2023evolutionary,
  title={Evolutionary Many-objective Optimization: Difficulties, Approaches, and Discussions},
  author={Sato, Hiroyuki and Ishibuchi, Hisao},
  journal={IEEJ Transactions on Electrical and Electronic Engineering},
  volume={18},
  number={7},
  pages={1048--1058},
  year={2023},
  publisher={Wiley Online Library}
}

@inproceedings{ding2024efficient,
  title={Efficient Algorithms for Sum-of-Minimum Optimization},
  author={Ding, Lisang and Chen, Ziang and Wang, Xinshang and Yin, Wotao},
  booktitle=icml,
  year={2024}
}

@inproceedings{li2025many,
  title={Many-Objective Multi-Solution Transport},
  author={Li, Ziyue and Li, Tian and Smith, Virginia and Bilmes, Jeff and Zhou, Tianyi},
  booktitle=iclr,
  year={2025}
}

@article{goffin1977convergence,
  title={On convergence rates of subgradient optimization methods},
  author={Goffin, Jean-Louis},
  journal={Mathematical programming},
  volume={13},
  pages={329--347},
  year={1977},
  publisher={Springer}
}

@article{nesterov2005smooth,
  title={Smooth minimization of non-smooth functions},
  author={Nesterov, Yu},
  journal={Mathematical Programming},
  volume={103},
  pages={127--152},
  year={2005},
  publisher={Springer}
}

@article{beck2012smoothing,
  title={Smoothing and first order methods: A unified framework},
  author={Beck, Amir and Teboulle, Marc},
  journal={SIAM Journal on Optimization},
  volume={22},
  number={2},
  pages={557--580},
  year={2012},
  publisher={SIAM}
}

@article{fliege2000steepest,
  title={Steepest descent methods for multicriteria optimization},
  author={Fliege, J{\"o}rg and Svaiter, Benar Fux},
  journal={Mathematical Methods of Operations Research},
  volume={51},
  number={3},
  pages={479--494},
  year={2000},
  publisher={Springer}
}

@article{schaffler2002stochastic,
  title={Stochastic method for the solution of unconstrained vector optimization problems},
  author={Sch{\"a}ffler, Stefan and Schultz, Reinhart and Weinzierl, Klaus},
  journal={Journal of Optimization Theory and Applications},
  volume={114},
  pages={209--222},
  year={2002},
  publisher={Springer}
}

@inproceedings{desideri2012mutiple,
  title={Mutiple-gradient descent algorithm for multiobjective optimization},
  author={D{\'e}sid{\'e}ri, Jean-Antoine},
  booktitle={European Congress on Computational Methods in Applied Sciences and Engineering (ECCOMAS 2012)},
  year={2012}
}

@article{geoffrion1967,
title = {Solving Bicriterion Mathematical Programs},
author = {Geoffrion, Arthur M.},
volume = {15},
number = {1},
journal = {Operations Research},
pages = {39–54},
year = {1967},
publisher = {INFORMS}
}

@incollection{bowman1976relationship,
  title={On the relationship of the Tchebycheff norm and the efficient frontier of multiple-criteria objectives},
  author={Bowman, V Joseph},
  booktitle={Multiple criteria decision making},
  pages={76--86},
  year={1976},
  publisher={Springer}
}

@article{choo1983proper,
  title={Proper efficiency in nonconvex multicriteria programming},
  author={Choo, Eng Ung and Atkins, DR},
  journal={Mathematics of Operations Research},
  volume={8},
  number={3},
  pages={467--470},
  year={1983},
  publisher={INFORMS}
}

@article{hillermeier2001generalized,
  title={Generalized homotopy approach to multiobjective optimization},
  author={Hillermeier, Claus},
  journal={Journal of Optimization Theory and Applications},
  volume={110},
  number={3},
  pages={557--583},
  year={2001},
  publisher={Springer}
}

@inproceedings{parisi2014policy,
  title={Policy gradient approaches for multi-objective sequential deb},
  author={ Parisi, Simone  and  Pirotta, Matteo  and  Smacchia, Nicola  and  Bascetta, Luca  and  Restelli, Marcello },
  booktitle={International Joint Conference on Neural Networks},
  year={2014},
}

@inproceedings{xu2020prediction,
  title={Prediction-Guided Multi-Objective Reinforcement Learning for Continuous Robot Control},
  author={Xu, Jie and Tian, Yunsheng and Ma, Pingchuan and Rus, Daniela and Sueda, Shinjiro and Matusik, Wojciech},
  booktitle={International Conference on Machine Learning},
  pages={10607--10616},
  year={2020},
  organization={PMLR}
}

@inproceedings{martinez2020minimax,
  title={Minimax pareto fairness: A multi objective perspective},
  author={Martinez, Natalia and Bertran, Martin and Sapiro, Guillermo},
  booktitle=icml,
  pages={6755--6764},
  year={2020},
  organization={PMLR}
}

@inproceedings{xie2020mars,
  title={MARS: Markov Molecular Sampling for Multi-objective Drug Discovery},
  author={Xie, Yutong and Shi, Chence and Zhou, Hao and Yang, Yuwei and Zhang, Weinan and Yu, Yong and Li, Lei},
  booktitle=iclr,
  year={2021}
}

@inproceedings{kingma2015adam,
  title={Adam: A Method for Stochastic Optimization},
  author={Kingma, Diederik P and Ba, Jimmy},
  booktitle=iclr,
  year={2015}
}

@article{matz2017psychological,
  title={Psychological targeting as an effective approach to digital mass persuasion},
  author={Matz, Sandra C and Kosinski, Michal and Nave, Gideon and Stillwell, David J},
  journal={Proceedings of the national academy of sciences},
  volume={114},
  number={48},
  pages={12714--12719},
  year={2017},
  publisher={National Acad Sciences}
}

@article{deb2006searching,
  title={Searching for Pareto-optimal solutions through dimensionality reduction for certain large-dimensional multi-objective optimization problems},
  author={ Deb, Kalyanmoy  and  Saxena, Dhish Kumar },
  journal={Computational Intelligence},
  pages={KanGAL Report Number 2005011},
  year={2005},
}

@article{brockhoff2006are,
  title={Are All Objectives Necessary? On Dimensionality Reduction in Evolutionary Multiobjective Optimization},
  author={ Brockhoff, Dimo  and  Zitzler, Eckart },
  journal={Lecture Notes in Computer Science},
  volume={4193},
  number={1},
  pages={533-542},
  year={2006},
}

@article{singh2011a,
  title={A Pareto Corner Search Evolutionary Algorithm and Dimensionality Reduction in Many-Objective Optimization Problems},
  author={ Singh, Hemant Kumar  and  Isaacs, Amitay  and  Ray, Tapabrata },
  journal={IEEE Transactions on Evolutionary Computation},
  volume={15},
  number={4},
  pages={539-556},
  year={2011},
}

@inproceedings{lin2025few,
  title={Few for Many: Tchebycheff Set Scalarization for Many-Objective Optimization},
  author={Lin, Xi and Liu, Yilu and Zhang, Xiaoyuan and Liu, Fei and Wang, Zhenkun and Zhang, Qingfu},
  booktitle=iclr,
  year={2025}
}

@article{liu2025few,
  title={Few for Many: Towards Efficient and Flexible Many-Objective Optimization},
  author={Liu, Yilu and Lin, Xi and Zhao, Liang and Zhang, Qingfu},
  journal={IEEE Transactions on Evolutionary Computation},
  year={2025},
  publisher={IEEE}
}

@inproceedings{liu2024many,
  title={Many-objective cover problem: Discovering few solutions to cover many objectives},
  author={Liu, Yilu and Lu, Chengyu and Lin, Xi and Zhang, Qingfu},
  booktitle=ppsn,
  pages={68--82},
  year={2024},
  organization={Springer}
}

@inproceedings{maus2025multi,
  title={Multi-Objective Coverage Bayesian Optimization (MOCOBO)},
  author={Maus, Natalie and Kim, Kyurae and Zeng, Yimeng and Jones, Haydn Thomas and Wan, Fangping and Torres, MDT and de la Fuente-Nunez, C and Gardner, JR},
  booktitle=nips,
  year={2025}
}

@inproceedings{lin2022pareto_moco,
  title={Pareto Set Learning for Neural Multi-objective Combinatorial Optimization},
  author={Lin, Xi and Yang, Zhiyuan and Zhang, Qingfu},
  booktitle=iclr,
  year={2022}
}

@inproceedings{lin2022pareto_expensive,
  title={Pareto set learning for expensive multiobjective optimization},
  author={Lin, Xi and Yang, Zhiyuan and Zhang, Xiaoyuan and Zhang, Qingfu},
  booktitle=nips,
  year={2022},
}

@inproceedings{ye2024evolutionary,
  title={Evolutionary Preference Sampling for {Pareto} Set Learning},
  author={Ye, Rongguang and Chen, Longcan and Zhang, Jinyuan and Ishibuchi, Hisao},
  booktitle=gecco,
  year={2024}
}

@article{secretan2011picbreeder,
  title={Picbreeder: A case study in collaborative evolutionary exploration of design space},
  author={Secretan, Jimmy and Beato, Nicholas and D'Ambrosio, David B and Rodriguez, Adelein and Campbell, Adam and Folsom-Kovarik, Jeremiah T and Stanley, Kenneth O},
  journal={Evolutionary computation},
  volume={19},
  number={3},
  pages={373--403},
  year={2011},
  publisher={MIT Press}
}

@article{doncieux2015evolutionary,
  title={Evolutionary robotics: what, why, and where to},
  author={Doncieux, Stephane and Bredeche, Nicolas and Mouret, Jean-Baptiste and Eiben, Agoston E},
  journal={Frontiers in Robotics and AI},
  volume={2},
  pages={4},
  year={2015},
  publisher={Frontiers Media SA}
}

@article{cully2015robots,
  title={Robots that can adapt like animals},
  author={Cully, Antoine and Clune, Jeff and Tarapore, Danesh and Mouret, Jean-Baptiste},
  journal={Nature},
  volume={521},
  number={7553},
  pages={503--507},
  year={2015},
  publisher={Nature Publishing Group UK London}
}

@article{ecoffet2019go,
  title={Go-explore: a new approach for hard-exploration problems},
  author={Ecoffet, Adrien and Huizinga, Joost and Lehman, Joel and Stanley, Kenneth O and Clune, Jeff},
  journal={arXiv preprint arXiv:1901.10995},
  year={2019}
}

@article{ecoffet2021first,
  title={First return, then explore},
  author={Ecoffet, Adrien and Huizinga, Joost and Lehman, Joel and Stanley, Kenneth O and Clune, Jeff},
  journal={Nature},
  volume={590},
  number={7847},
  pages={580--586},
  year={2021},
  publisher={Nature Publishing Group UK London}
}

@article{nguyen2016understanding,
  title={Understanding innovation engines: Automated creativity and improved stochastic optimization via deep learning},
  author={Nguyen, Anh and Yosinski, Jason and Clune, Jeff},
  journal={Evolutionary computation},
  volume={24},
  number={3},
  pages={545--572},
  year={2016},
  publisher={MIT Press One Rogers Street, Cambridge, MA 02142-1209, USA journals-info~…}
}

@inproceedings{tian2022modern,
  title={Modern evolution strategies for creativity: Fitting concrete images and abstract concepts},
  author={Tian, Yingtao and Ha, David},
  booktitle={International conference on computational intelligence in music, sound, art and design (part of evostar)},
  pages={275--291},
  year={2022},
  organization={Springer}
}

@inproceedings{bradley2023quality,
  title={Quality-diversity through AI feedback},
  author={Bradley, Herbie and Dai, Andrew and Teufel, Hannah and Zhang, Jenny and Oostermeijer, Koen and Bellagente, Marco and Clune, Jeff and Stanley, Kenneth and Schott, Gr{\'e}gory and Lehman, Joel},
  booktitle=iclr,
  year={2023}
}

@inproceedings{ding2024quality,
  title={Quality Diversity through Human Feedback: Towards Open-Ended Diversity-Driven Optimization},
  author={Ding, Li and Zhang, Jenny and Clune, Jeff and Spector, Lee and Lehman, Joel},
  booktitle=icml,
  year={2024}
}

@inproceedings{gonzalez2020finding,
  title={Finding game levels with the right difficulty in a few trials through intelligent trial-and-error},
  author={Gonz{\'a}lez-Duque, Miguel and Palm, Rasmus Berg and Ha, David and Risi, Sebastian},
  booktitle={2020 IEEE Conference on Games (CoG)},
  pages={503--510},
  year={2020},
  organization={IEEE}
}

@article{bhatt2022deep,
  title={Deep surrogate assisted generation of environments},
  author={Bhatt, Varun and Tjanaka, Bryon and Fontaine, Matthew and Nikolaidis, Stefanos},
  journal=nips,
  year={2022}
}

@article{fontaine2022evaluating,
  title={Evaluating human--robot interaction algorithms in shared autonomy via quality diversity scenario generation},
  author={Fontaine, Matthew C and Nikolaidis, Stefanos},
  journal={ACM Transactions on Human-Robot Interaction (THRI)},
  volume={11},
  number={3},
  pages={1--30},
  year={2022},
  publisher={ACM New York, NY}
}

@article{zhang2023arbitrarily,
  title={Arbitrarily scalable environment generators via neural cellular automata},
  author={Zhang, Yulun and Fontaine, Matthew and Bhatt, Varun and Nikolaidis, Stefanos and Li, Jiaoyang},
  journal=nips,
  year={2023}
}

@inproceedings{nguyen2015deep,
  title={Deep neural networks are easily fooled: High confidence predictions for unrecognizable images},
  author={Nguyen, Anh and Yosinski, Jason and Clune, Jeff},
  booktitle={Proceedings of the IEEE conference on computer vision and pattern recognition},
  pages={427--436},
  year={2015}
}

@article{samvelyan2024rainbow,
  title={Rainbow teaming: Open-ended generation of diverse adversarial prompts},
  author={Mikayel Samvelyan and Sharath Chandra Raparthy and Andrei Lupu and Eric Hambro and Aram H. Markosyan and Manish Bhatt and Yuning Mao and Minqi Jiang and Jack Parker-Holder and Jakob Foerster and Tim Rocktäschel and Roberta Raileanu},
  journal={Advances in Neural Information Processing Systems},
  volume={37},
  pages={69747--69786},
  year={2024}
}

@article{wang2025quality,
  title={Quality-Diversity Red-Teaming: Automated Generation of High-Quality and Diverse Attackers for Large Language Models},
  author={Wang, Ren-Jian and Xue, Ke and Qin, Zeyu and Li, Ziniu and Tang, Sheng and Li, Hao-Tian and Liu, Shengcai and Qian, Chao},
  journal={arXiv preprint arXiv:2506.07121},
  year={2025}
}

@misc{novikov2025alphaevolve,
      title={AlphaEvolve: A coding agent for scientific and algorithmic discovery}, 
      author={Alexander Novikov and Ngân Vũ and Marvin Eisenberger and Emilien Dupont and Po-Sen Huang and Adam Zsolt Wagner and Sergey Shirobokov and Borislav Kozlovskii and Francisco J. R. Ruiz and Abbas Mehrabian and M. Pawan Kumar and Abigail See and Swarat Chaudhuri and George Holland and Alex Davies and Sebastian Nowozin and Pushmeet Kohli and Matej Balog},
      year={2025},
      eprint={2506.13131},
      archivePrefix={arXiv},
      primaryClass={cs.AI},
      url={https://arxiv.org/abs/2506.13131}, 
}

@article{cully2017quality,
  title={Quality and diversity optimization: A unifying modular framework},
  author={Cully, Antoine and Demiris, Yiannis},
  journal={IEEE Transactions on Evolutionary Computation},
  volume={22},
  number={2},
  pages={245--259},
  year={2017},
  publisher={IEEE}
}

@incollection{chatzilygeroudis2021quality,
  title={Quality-diversity optimization: a novel branch of stochastic optimization},
  author={Chatzilygeroudis, Konstantinos and Cully, Antoine and Vassiliades, Vassilis and Mouret, Jean-Baptiste},
  booktitle={Black box optimization, machine learning, and no-free lunch theorems},
  pages={109--135},
  year={2021},
  publisher={Springer}
}

@article{pugh2016quality,
  title={Quality diversity: A new frontier for evolutionary computation},
  author={Pugh, Justin K and Soros, Lisa B and Stanley, Kenneth O},
  journal={Frontiers in Robotics and AI},
  volume={3},
  pages={40},
  year={2016},
  publisher={Frontiers Media SA}
}

@article{mouret2015illuminating,
  title={Illuminating search spaces by mapping elites},
  author={Mouret, Jean-Baptiste and Clune, Jeff},
  journal={arXiv preprint arXiv:1504.04909},
  year={2015}
}

@article{vassiliades2017using,
  title={Using centroidal voronoi tessellations to scale up the multidimensional archive of phenotypic elites algorithm},
  author={Vassiliades, Vassilis and Chatzilygeroudis, Konstantinos and Mouret, Jean-Baptiste},
  journal={IEEE Transactions on Evolutionary Computation},
  volume={22},
  number={4},
  pages={623--630},
  year={2017},
  publisher={IEEE}
}

@inproceedings{gaier2017data,
  title={Data-efficient exploration, optimization, and modeling of diverse designs through surrogate-assisted illumination},
  author={Gaier, Adam and Asteroth, Alexander and Mouret, Jean-Baptiste},
  booktitle={Proceedings of the Genetic and Evolutionary Computation Conference},
  pages={99--106},
  year={2017}
}

@article{kent2020bop,
  title={Bop-elites, a bayesian optimisation algorithm for quality-diversity search},
  author={Kent, Paul and Branke, Juergen},
  journal={arXiv preprint arXiv:2005.04320},
  year={2020}
}

@inproceedings{zhang2022deep,
  title={Deep surrogate assisted map-elites for automated hearthstone deckbuilding},
  author={Zhang, Yulun and Fontaine, Matthew C and Hoover, Amy K and Nikolaidis, Stefanos},
  booktitle={Proceedings of the Genetic and Evolutionary Computation Conference},
  pages={158--167},
  year={2022}
}

@inproceedings{mouret2020quality,
  title={Quality diversity for multi-task optimization},
  author={Mouret, Jean-Baptiste and Maguire, Glenn},
  booktitle={Proceedings of the 2020 Genetic and Evolutionary Computation Conference},
  pages={121--129},
  year={2020}
}

@inproceedings{pugh2015confronting,
  title={Confronting the challenge of quality diversity},
  author={Pugh, Justin K and Soros, Lisa B and Szerlip, Paul A and Stanley, Kenneth O},
  booktitle={Proceedings of the 2015 annual conference on genetic and evolutionary computation},
  pages={967--974},
  year={2015}
}

@inproceedings{fontaine2020covariance,
  title={Covariance matrix adaptation for the rapid illumination of behavior space},
  author={Fontaine, Matthew C and Togelius, Julian and Nikolaidis, Stefanos and Hoover, Amy K},
  booktitle=gecco,
  pages={94--102},
  year={2020}
}

@inproceedings{cully2018hierarchical,
  title={Hierarchical behavioral repertoires with unsupervised descriptors},
  author={Cully, Antoine and Demiris, Yiannis},
  booktitle={Proceedings of the Genetic and Evolutionary Computation Conference},
  pages={69--76},
  year={2018}
}

@article{grillotti2022unsupervised,
  title={Unsupervised behavior discovery with quality-diversity optimization},
  author={Grillotti, Luca and Cully, Antoine},
  journal={IEEE Transactions on Evolutionary Computation},
  volume={26},
  number={6},
  pages={1539--1552},
  year={2022},
  publisher={IEEE}
}

@inproceedings{mouret2023fast,
  title={Fast generation of centroids for MAP-Elites},
  author={Mouret, Jean-Baptiste},
  booktitle={Proceedings of the Companion Conference on Genetic and Evolutionary Computation},
  pages={155--158},
  year={2023}
}

@inproceedings{lehman2008exploiting,
  title={Exploiting open-endedness to solve problems through the search for novelty.},
  author={Lehman, Joel and Stanley, Kenneth O},
  booktitle={ALIFE},
  pages={329--336},
  year={2008}
}

@article{lehman2011abandoning,
  title={Abandoning objectives: Evolution through the search for novelty alone},
  author={Lehman, Joel and Stanley, Kenneth O},
  journal={Evolutionary computation},
  volume={19},
  number={2},
  pages={189--223},
  year={2011},
  publisher={MIT Press}
}

@inproceedings{lehman2011evolving,
  title={Evolving a diversity of virtual creatures through novelty search and local competition},
  author={Lehman, Joel and Stanley, Kenneth O},
  booktitle={Proceedings of the 13th annual conference on Genetic and evolutionary computation},
  pages={211--218},
  year={2011}
}

@inproceedings{kent2022discretization,
  title={A discretization-free metric for assessing quality diversity algorithms},
  author={Kent, Paul and Branke, Juergen and Gaier, Adam and Mouret, Jean-Baptiste},
  booktitle={Proceedings of the Genetic and Evolutionary Computation Conference Companion},
  pages={2131--2135},
  year={2022}
}

@inproceedings{nilsson2021policy,
  title={Policy gradient assisted map-elites},
  author={Nilsson, Olle and Cully, Antoine},
  booktitle={Proceedings of the Genetic and Evolutionary Computation Conference},
  pages={866--875},
  year={2021}
}

@article{fontaine2021differentiable,
  title={Differentiable quality diversity},
  author={Fontaine, Matthew and Nikolaidis, Stefanos},
  journal={Advances in Neural Information Processing Systems},
  volume={34},
  pages={10040--10052},
  year={2021}
}

@inproceedings{fontaine2023covariance,
  title={Covariance matrix adaptation map-annealing},
  author={Fontaine, Matthew and Nikolaidis, Stefanos},
  booktitle={Proceedings of the genetic and evolutionary computation conference},
  pages={456--465},
  year={2023}
}

@article{tjanaka2023training,
  title={Training diverse high-dimensional controllers by scaling covariance matrix adaptation map-annealing},
  author={Tjanaka, Bryon and Fontaine, Matthew C and Lee, David H and Kalkar, Aniruddha and Nikolaidis, Stefanos},
  journal={IEEE Robotics and Automation Letters},
  volume={8},
  number={10},
  pages={6771--6778},
  year={2023},
  publisher={IEEE}
}

@inproceedings{batra2024proximal,
  title={Proximal Policy Gradient Arborescence for Quality Diversity Reinforcement Learning},
  author={Batra, Sumeet and Tjanaka, Bryon and Fontaine, Matthew Christopher and Petrenko, Aleksei and Nikolaidis, Stefanos and Sukhatme, Gaurav S},
  booktitle=iclr,
  year = {2024}
}

@inproceedings{wan2025diversifying,
  title={Diversifying Robot Locomotion Behaviors with Extrinsic Behavioral Curiosity},
  author={Wan, Zhenglin and Yu, Xingrui and Bossens, David Mark and Lyu, Yueming and Guo, Qing and Fan, Flint Xiaofeng and Ong, Yew-Soon and Tsang, Ivor},
  booktitle=icml,
  year = 2025
}

@inproceedings{bahlous2025dominated,
  title={Dominated novelty search: Rethinking local competition in quality-diversity},
  author={Bahlous-Boldi, Ryan and Faldor, Maxence and Grillotti, Luca and Janmohamed, Hannah and Coiffard, Lisa and Spector, Lee and Cully, Antoine},
  booktitle={Proceedings of the Genetic and Evolutionary Computation Conference},
  pages={104--112},
  year={2025}
}

@article{friedman2023vendi,
  title={The Vendi Score: A Diversity Evaluation Metric for Machine Learning},
  author={Friedman, Dan and Dieng, Adji Bousso},
  journal={Transactions on Machine Learning Research},
  year={2023}
}

@inproceedings{nguyen2024quality,
  title={Quality-weighted vendi scores and their application to diverse experimental design},
  author={Nguyen, Quan and Dieng, Adji Bousso},
  booktitle=icml,
  year={2024}
}

@inproceedings{mouret2011novelty,
  title={Novelty-based multiobjectivization},
  author={Mouret, Jean-Baptiste},
  booktitle={New Horizons in Evolutionary Robotics: Extended Contributions from the 2009 EvoDeRob Workshop},
  pages={139--154},
  year={2011},
  organization={Springer}
}

@inproceedings{pierrot2022multi,
  title={Multi-objective quality diversity optimization},
  author={Pierrot, Thomas and Richard, Guillaume and Beguir, Karim and Cully, Antoine},
  booktitle={Proceedings of the genetic and evolutionary computation conference},
  pages={139--147},
  year={2022}
}

@inproceedings{janmohamed2023improving,
  title={Improving the data efficiency of multi-objective quality-diversity through gradient assistance and crowding exploration},
  author={Janmohamed, Hannah and Pierrot, Thomas and Cully, Antoine},
  booktitle={Proceedings of the Genetic and Evolutionary Computation Conference},
  pages={165--173},
  year={2023}
}

@inproceedings{janmohamed2025multi,
  title={Multi-objective quality-diversity in unstructured and unbounded spaces},
  author={Janmohamed, Hannah and Cully, Antoine},
  booktitle={Proceedings of the Genetic and Evolutionary Computation Conference},
  pages={149--157},
  year={2025}
}

@inproceedings{zhao2025multi,
  title={Multi-Objective Covariance Matrix Adaptation MAP-Annealing},
  author={Zhao, Shihan and Nikolaidis, Stefanos},
  booktitle={Proceedings of the Genetic and Evolutionary Computation Conference},
  pages={673--682},
  year={2025}
}

@inproceedings{janmohamed2024multi,
  title={Multi-Objective Quality-Diversity for Crystal Structure Prediction},
  author={Janmohamed, Hannah and Wolinska, Marta and Surana, Shikha and Pierrot, Thomas and Walsh, Aron and Cully, Antoine},
  booktitle={Proceedings of the Genetic and Evolutionary Computation Conference},
  pages={1273--1281},
  year={2024}
}

@inproceedings{shen2020generating,
  title={Generating behavior-diverse game ais with evolutionary multi-objective deep reinforcement learning},
  author={Shen, Ruimin and Zheng, Yan and Hao, Jianye and Meng, Zhaopeng and Chen, Yingfeng and Fan, Changjie and Liu, Yang},
  booktitle={Proceedings of the Twenty-Ninth International Conference on International Joint Conferences on Artificial Intelligence},
  pages={3371--3377},
  year={2020}
}

@inproceedings{villin2021effects,
  title={Effects of different optimization formulations in evolutionary reinforcement learning on diverse behavior generation},
  author={Villin, Victor and Masuyama, Naoki and Nojima, Yusuke},
  booktitle={2021 IEEE Symposium Series on Computational Intelligence (SSCI)},
  year={2021},
  organization={IEEE}
}

@inproceedings{wang2023multi,
  title={Multi-objective Optimization-based Selection for Quality-Diversity by Non-surrounded-dominated Sorting.},
  author={Wang, Ren-Jian and Xue, Ke and Shang, Haopu and Qian, Chao and Fu, Haobo and Fu, Qiang},
  booktitle={IJCAI},
  pages={4335--4343},
  year={2023}
}

@article{hedayatian2025autoqd,
  title={AutoQD: Automatic Discovery of Diverse Behaviors with Quality-Diversity Optimization},
  author={Hedayatian, Saeed and Nikolaidis, Stefanos},
  journal={arXiv preprint arXiv:2506.05634},
  year={2025}
}

@article{hedayatian2025soft,
  title={Soft Quality-Diversity Optimization},
  author={Hedayatian, Saeed and Nikolaidis, Stefanos},
  journal={arXiv preprint arXiv:2512.00810},
  year={2025}
}

@inproceedings{tjanaka2023pyribs,
  title={pyribs: A bare-bones python library for quality diversity optimization},
  author={Tjanaka, Bryon and Fontaine, Matthew C and Lee, David H and Zhang, Yulun and Balam, Nivedit Reddy and Dennler, Nathaniel and Garlanka, Sujay S and Klapsis, Nikitas Dimitri and Nikolaidis, Stefanos},
  booktitle={Proceedings of the Genetic and Evolutionary Computation Conference},
  pages={220--229},
  year={2023}
}

@article{rastrigin1974systems,
  title={Systems of extremal control},
  author={Rastrigin, Leonard Andreevi{\v{c}}},
  journal={Nauka},
  year={1974}
}

@inproceedings{hoffmeister1990genetic,
  title={Genetic algorithms and evolution strategies: Similarities and differences},
  author={Hoffmeister, Frank and B{\"a}ck, Thomas},
  booktitle={International conference on parallel problem solving from nature},
  pages={455--469},
  year={1990},
  organization={Springer}
}

@article{wang2024image,
  title={Image quality assessment: from error visibility to structural similarity}, 
  author={Zhou Wang and Bovik, A.C. and Sheikh, H.R. and Simoncelli, E.P.},
  journal={IEEE Transactions on Image Processing}, 
  year={2004},
  volume={13},
  number={4},
  pages={600-612}
  }

@inproceedings{karras2020analyzing,
  title={Analyzing and improving the image quality of stylegan},
  author={Karras, Tero and Laine, Samuli and Aittala, Miika and Hellsten, Janne and Lehtinen, Jaakko and Aila, Timo},
  booktitle=cvpr,
  pages={8110--8119},
  year={2020}
}

@inproceedings{radford2021learning,
  title={Learning transferable visual models from natural language supervision},
  author={Alec Radford and Jong Wook Kim and Chris Hallacy and Aditya Ramesh and Gabriel Goh and Sandhini Agarwal and Girish Sastry and Amanda Askell and Pamela Mishkin and Jack Clark and Gretchen Krueger and Ilya Sutskever},
  booktitle=icml,
  pages={8748--8763},
  year={2021}
}

\newpage
\appendix
\onecolumn




In this appendix, we mainly provide:
\begin{itemize}

    \item \textbf{Background of Multi-Objective Optimization} can be found in Section~\ref{supp_sec_background}.
    
    \item \textbf{Detailed Proofs} for the theoretical analysis are provide in Section~\ref{supp_sec_proof}. 

    \item \textbf{Derivation of STCH-Set} can be found in Section~\ref{supp_sec_ssom_stch-set}.  

    \item \textbf{Problem Details} can be found in Section~\ref{supp_sec_problem}.
    
    \item \textbf{Experimental Settings} can be found in Section~\ref{supp_sec_setting}.

    \item \textbf{More Experimental Results and Analyses} are provided in Section~\ref{supp_sec_experiment}.
    
\end{itemize}

\section{Background of Multi-Objective Optimization}
\label{supp_sec_background}

\paragraph{Multi-Objective Optimization Problem}

Practical Optimization problems usually involving multiple, often competing, objectives to optimize. These include training machine learning models that balance accuracy with fairness~\cite{martinez2020minimax}, developing artificial agents that maximize performance while minimizing energy consumption~\cite{xu2020prediction}, and discovering molecular structures that satisfy multiple pharmaceutical criteria simultaneously~\cite{xie2020mars}. Mathematically, such problems are formulated as multi-objective optimization problems (MOPs), where the goal is to optimize several objective functions concurrently. Since these objectives typically conflict, no single solution can optimize all of them simultaneously. Instead, the solution set consists of Pareto optimal points, each representing a different compromise among objectives, collectively forming the Pareto set and its image, the Pareto front in the objective space~\cite{miettinen1999nonlinear, ehrgott2005multicriteria, ehrgott2025fifty}.

\paragraph{Multi-Objective Evolutionary Algorithm}

For black-box optimization problems where gradient information are unavailable, evolutionary algorithms have been a dominant methodology. Multi-objective evolutionary algorithms (MOEAs) maintain a population of solutions, using mechanisms inspired by natural selection (crossover and mutation) to evolve them toward the Pareto set. Classic algorithms like NSGA-II~\cite{deb2002fast}, IBEA~\cite{zitzler2004indicator}, SMS-EMOA~\cite{beume2007sms},  MOEA/D~\cite{zhang2007moea} and most of their variants are effective for problems with two or three objectives. However, their performance degrades as the number of objectives increases, a challenge known as many-objective optimization~\cite{fleming2005many}. The exponential growth in the portion of non-dominated solutions complicates selection and diversity preservation~\citep{purshouse2007evolutionary,knowles2007quantifying}. While specialized many-objective evolutionary algorithms~\cite{bader2011hype, li2015many, sato2023evolutionary} and objective reduction techniques~\cite{deb2006searching, brockhoff2006are, singh2011a} have been proposed, they often remain computationally expensive and struggle with high-dimensional objective spaces.

\paragraph{Gradient-based Multi-Objective Optimization}

When the objective functions are differentiable, gradient-based methods offer a more efficient optimization approach. These methods are broadly categorized into scalarization approaches and adaptive gradient algorithms. Scalarization methods aggregate multiple objectives into a single function using a preference vector. The simplest linear scalarization~\cite{geoffrion1967} is efficient but is theoretically incapable of finding Pareto optimal solutions that lie on non-convex regions of the front~\cite{das1997a, boyd2004convex}. In contrast, the Tchebycheff (TCH) scalarization method~\cite{bowman1976relationship} is a desirable alternative. Under mild conditions, any (weakly) Pareto optimal solution can be obtained by solving a TCH problem with an appropriate preference vector~\cite{choo1983proper}. However, its non-smooth $\max$ operator leads to challenges in gradient-based optimization, often resulting in slow convergence~\cite{goffin1977convergence}. Recent work has introduced a smooth approximation to the TCH scalarization (STCH)~\cite{lin2024smooth}, which enables faster gradient descent while maintaining strong theoretical guarantees for multi-objective optimization. This approach is also found beneficial in robust optimization~\cite{he2024robust} and reinforcement learning settings~\cite{qiu2024traversing}.

Adaptive gradient methods, such as the Multiple Gradient Descent Algorithm (MGDA) \cite{fliege2000steepest, schaffler2002stochastic, desideri2012mutiple}, take a different approach. At each gradient-based optimization step, they compute a search direction that is a descent direction for all objectives simultaneously. In recent years, these methods have been adopted to tackle different multi-objective machine learning problems~\cite{sener2018multi,lin2019pareto,mahapatramulti2020multi,ma2020efficient, momma2022multi}. However, these methods incur significant computational overhead per iteration, as they require separate gradient computations for each objective and solving a quadratic programming (QP) problem. This complexity becomes a major bottleneck for problems with many objectives.

\paragraph{Pareto Set Learning}

A distinct line of research seeks to model the entire Pareto set as a continuous manifold rather than finding a discrete set of points~\cite{parisi2014policy,dosovitskiy2019you, lin2020controllable,navon2021learning}. Inspired by the observation that the Pareto set often forms a low-dimensional structure~\cite{hillermeier2001generalized}, Pareto Set Learning (PSL) methods train a parametric model (e.g., a neural network) that maps a preference vector directly to a corresponding Pareto optimal solution. This framework allows for efficient, continuous traversal of the trade-off space after training. PSL has been successfully applied to domains such as multi-task learning~\cite{lin2020controllable, ruchte2021scalable}, neural combinatorial optimization~\cite{lin2022pareto_moco}, and expensive black-box optimization~\cite{lin2022pareto_expensive}. Recent advances focus on improving training efficiency through bi-level optimization~\cite{chen2022multi}, evolutionary sampling~\cite{ye2024evolutionary}, and integrating smooth scalarization techniques~\cite{lin2024smooth}. A fundamental limitation of standard PSL, however, is its assumption that a dense sampling of preferences is needed to learn the model, which becomes infeasible when the number of objectives is very large.

\paragraph{Few Solutions for Many Optimization Objectives}

The practical scenario where a very large number of objectives must be addressed by a relatively small, fixed set of solutions has received limited attention in the multi-objective optimization literature. This setting is highly relevant in applications like product design with numerous performance metrics~\cite{fleming2005many}, creating a handful of advertising variants for a diverse audience~\cite{matz2017psychological}, or training a compact ensemble of models to handle many tasks~\cite{standley2020tasks, fifty2021efficiently}. The core challenge is to ensure that the small solution set collectively provides good coverage for all objectives. Recently, a few work has begun to address this gap. \citet{ding2024efficient} study a related sum-of-minimums formulation for problems like mixed linear regression, though not framed as a multi-objective problem. \citet{li2025many} propose a bi-level optimal transport framework to partition the objective space and assign representative subproblems to different solutions. \citet{lin2025few} propose a novel and efficient set scalarization approach that directly optimizes the collective performance of a small solution set over many objectives. These setting has also been generalized to tackle black-box optimization~\cite{liu2024many,liu2025few,maus2025multi}.

\paragraph{Multi-Objective Quality-Diversity Optimization} The concept of multi-objective optimization has also been integrated into quality-diversity (QD) optimization. For instance, a pioneering work~\cite{mouret2011novelty} treats the balance between fitness (quality) and novelty (diversity) as a two-objective problem, addressing it with multi-objective evolutionary algorithms. The Multi-Objective Quality-Diversity (MOQD) paradigm~\cite{pierrot2022multi} was subsequently proposed to find Pareto solutions among multiple objectives within each cell of a MAP-Elites archive~\cite{pierrot2022multi, janmohamed2023improving, janmohamed2024multi, zhao2025multi} or within each local region of an unstructured archive~\cite{janmohamed2025multi}. Other studies have employed multi-objective techniques to generate behaviorally diverse solutions for single-objective QD problems~\cite{shen2020generating, villin2021effects, wang2023multi}. In contrast to these approaches, our work reformulates QD optimization itself as a multi-objective problem defined by a large set of objectives that collectively cover the entire behavior space.

\clearpage
\section{Detailed Proof and Discussion}
\label{supp_sec_proof}

\subsection{Proof of Theorem~\ref{thm_informal_monotonicity}}

The monotonicity of the SoM and TCH-Set scalarizations has been established and proved by \citet{liu2025few}. This work completes the theoretical picture by proving that monotonicity is preserved under their smooth approximations, i.e., for the SSoM and STCH-Set scalarizations. An important corollary of our analysis is that the STCH-Set scalarization satisfies monotonicity under weaker conditions, no longer requiring the reference points $z_m^*$ to be equal as for TCH-Set scalarization.

\textbf{Theorem~\ref{thm_informal_monotonicity} }(Monotonicity).
\textit{Let $g$ be a function among $\{g^{\text{SoM}}, g_{\mu}^{\text{SSoM}}, g_{\mu}^{\text{STCH-Set}}\}$. For any two solution sets $\vX_U \subseteq \vX_W \subseteq \mathcal{X}$ with $1 \leq U \leq W$, it holds that $g(\vX_U) \geq g(\vX_W)$. Besides, if all reference points are equal ($z_1^* = \cdots = z_M^*$), $g^{\text{TCH-Set}}(\vX_U) \geq g^{\text{TCH-Set}}(\vX_W)$ also holds.} 
\vspace{-0.1in}
\begin{proof}

We first demonstrate the monotonicity of $g^{\text{SoM}}(\vX_K)$ and $g^{\text{TCH-Set}}(\vX_K)$. 

For simplicity, we use $\tilde{V}_{K,m}^*$ to denote $\min_{\vx \in \vX_K} \tilde{v}_m(\vx)$. 
We first show that $\tilde{V}_{U,m}^* \geq \tilde{V}_{W,m}^*$ holds for all indexes $m\in\left\{1,2,\dots,M\right\}$. For simplicity, we use $j^*$ to denote the index of the best solution for $\tilde{v}_m(\vx)$ within $\vX_W$, i.e., $j^*=\argmin_{1 \leq j \leq W}\tilde{v}_m(\vx_j)$. Therefore, we only need to discuss the following two cases: 
\begin{enumerate}
\item $1 \leq j^* \leq U$. This case means that the best solution for $\tilde{v}_m(\vx)$ is in $\vX_U$. As $\vX_U \subseteq \vX_W$, we can derive $\tilde{V}_{U,m}^* = \tilde{V}_{W,m}^*$.
\item $U < j^* \leq W$. This case means that the best solution for $\tilde{v}_m(\vx)$ is in $\vX_W$\textbackslash $\vX_U$, showing that $\tilde{V}_{U,m}^* \geq \tilde{V}_{W,m}^*$.
\end{enumerate}
According to these two cases, we have $\tilde{V}_{U,m}^* \geq \tilde{V}_{W,m}^*$. With this conclusion, we can easily derive the following two facts:
\begin{enumerate}
    \item As $g_{\mu}^{\text{SoM}}(\vX_K)$ is a convex combination of $\tilde{V}_{K,m}^*$, we have $g_{\mu}^{\text{SoM}}(\vX_U) \geq g_{\mu}^{\text{SoM}}(\vX_W)$.
    \item Let $z^*=z_1^*=z_2^*=\cdots=z_m^*$, $m_U^*=\argmax_{1\leq m\leq M}(\tilde{V}_{U,m}^*-z^*)$, and $m_W^*=\argmax_{1\leq m\leq M}(\tilde{V}_{W,m}^*-z^*)$. We have 
    \begin{equation}
        \tilde{V}_{U,m_U^*}^*-z^*\geq\tilde{V}_{U,m_W^*}^*-z^*.
    \end{equation}
    Since $\tilde{V}_{U,m}^* \geq \tilde{V}_{W,m}^*$, we also have 
    \begin{equation}
        \tilde{V}_{U,m_W^*}^*-z^*\geq\tilde{V}_{W,m_W^*}^*-z^*.
    \end{equation}
    The above two relations yield
    \begin{equation}
        \tilde{V}_{U,m_U^*}^*-z^*\geq\tilde{V}_{W,m_W^*}^*-z^*,
    \end{equation}
    which shows that $g^{\text{TCH-Set}}(\vX_U) \geq g^{\text{TCH-Set}}(\vX_W)$.
\end{enumerate}
With these two facts, the monotonicity of $g_{\mu}^{\text{SoM}}(\vX_K)$ and $g_{\mu}^{\text{TCH-Set}}(\vX_K)$ has been confirmed. Next, we demonstrate the monotonicity of $g_{\mu}^{\text{SSoM}}(\vX_K)$ and $g_{\mu}^{\text{STCH-Set}}(\vX_K)$.

First, we have
\begin{align}
    g_{\mu}^{\text{SSoM}}(\vX_W) &= -\sum_{m=1}^{M} \mu \log\left[ \left(\sum_{k=1}^{U} e^{-\frac{\tilde{v}_m(\vx^{(k)})}{\mu}} \right) + \left(\sum_{k=U+1}^{W} e^{-\frac{\tilde{v}_m(\vx^{(k)})}{\mu}}\right)\right] \nonumber \\
    &\leq -\sum_{m=1}^{M} \mu \log\left( \sum_{k=1}^{U} e^{-\frac{\tilde{v}_m(\vx^{(k)})}{\mu}} \right) =g_{\mu}^{\text{SSoM}}(\vX_U).
\end{align}
Similarly, we have 
\begin{align}
        g_{\mu}^{\text{STCH-Set}}(\vX_W) 
    &= \mu \log\left( \sum_{m=1}^{M} e^{  - \log\left[ \left(\sum_{k=1}^{U} e^{- \frac{\tilde{v}_m(\vx^{(k)})}{\mu}}\right) + \left(\sum_{k=U+1}^{W} e^{- \frac{\tilde{v}_m(\vx^{(k)})}{\mu}}\right) \right]    - z_m^*} \right) \nonumber \\
    &\leq \mu \log\left( \sum_{m=1}^{M} e^{  - \log \left(\sum_{k=1}^{U} e^{- \frac{\tilde{v}_m(\vx^{(k)})}{\mu}}\right)     - z_m^*} \right)=g_{\mu}^{\text{STCH-Set}}(\vX_U).
\end{align}
Therefore, Theorem~\ref{thm_informal_monotonicity} holds.
\end{proof}

\subsection{Proof of Theorem~\ref{thm_informal_supermodularity}}

Similarly, the supermodularity of the SoM and TCH-Set scalarizations has also been analyzed by \citet{liu2025few}. We extend this analysis to their smooth approximations, proving that the SSoM and STCH-Set scalarizations also satisfy supermodularity. An important finding is that the STCH-Set scalarization retains this property under a relaxed condition. Specifically, it does not require the reference points $z_m^*$ to be equal, nor does it require the alignment condition on the $\argmax$ operator that is needed for the non-smooth TCH-Set scalarization.

\textbf{Theorem~\ref{thm_informal_supermodularity} }(Supermodularity). 
\textit{Let $g$ be a function among $\{g^{\text{SoM}}, g_{\mu}^{\text{SSoM}}, g_{\mu}^{\text{STCH-Set}}\}$. For any two solution sets $\vX_U \subseteq \vX_W \subseteq \mathcal{X}$ with $1 \leq U \leq W$ and any solution $\boldsymbol{x}' \in \mathcal{X}$\textbackslash $\vX_W$, it holds that $g(\vX_U) - g(\vX_U \cup \left\{\boldsymbol{x}' \right\}) \geq g(\vX_W) - g(\vX_W \cup \left\{\boldsymbol{x}' \right\})$. Besides, if all reference points are equal ($z_1^* = \cdots = z_M^*$) and $\argmax_{1\leq m\leq M}(\min_{\vx \in \vX_U\cup \left\{\boldsymbol{x}' \right\}} \tilde{v}_m(\vx) -z_m^*)=\argmax_{1\leq m\leq M}(\min_{\vx \in \vX_W} \tilde{v}_m(\vx)-z_m^*)$, the same inequality holds for $g^{\text{TCH-Set}}$.} 
\begin{proof}
Let $\tilde{V}_{K,m}^*=\min_{\vx \in \vX_K} \tilde{v}_m(\vx)$, $\vX_U'=\vX_U \cup \left\{\boldsymbol{x}'\right\}$, and $U'=|\vX_U \cup \left\{\boldsymbol{x}' \right\}|$. We first show that $\tilde{V}_{U,m}^*-\tilde{V}_{U',m}^*\geq \tilde{V}_{W,m}^*-\tilde{V}_{W',m}^*$. For simplicity, we use $j^*$ to denote the index of the best solution for $\tilde{v}_m(\vx)$ within $\vX_W'$, i.e., $j^*=\argmin_{1 \leq j \leq W+1}\tilde{v}_m(\boldsymbol{x}_j)$. Therefore, we only need to discuss the following three cases: 
\begin{enumerate}
\item $1 \leq j^* \leq U$. This case means that the best solution for $\tilde{v}_m(\boldsymbol{x})$ is in $\vX_U$. Thus, we can derive $\tilde{V}_{U,m}^* = \tilde{V}_{U',m}^* = \tilde{V}_{W,m}^* = \tilde{V}_{W',m}^*$, showing that $\tilde{V}_{U,m}^*-\tilde{V}_{U',m}^*\geq \tilde{V}_{W,m}^*-\tilde{V}_{W',m}^*$.
\item $U < j^* \leq W$. This case means that the best solution for $\tilde{v}_m(\boldsymbol{x})$ is in $\vX_W$\textbackslash $\vX_U$, showing that $\tilde{V}_{W,m}^* = \tilde{V}_{W',m}^*$. Since $\tilde{V}_{U,m}^*\geq \tilde{V}_{U',m}^*$, we have $\tilde{V}_{U,m}^*-\tilde{V}_{U',m}^*\geq \tilde{V}_{W,m}^*-\tilde{V}_{W',m}^*$.
\item $j^* = W+1$. This case means that the best solution for $\tilde{v}_m(\boldsymbol{x})$ is $\boldsymbol{x}'$, showing that $\tilde{V}_{U',m}^* = \tilde{V}_{W',m}^*$. Since $\tilde{V}_{U,m}^*\geq \tilde{V}_{W,m}^*$, we have $\tilde{V}_{U,m}^*-\tilde{V}_{U',m}^*\geq \tilde{V}_{W,m}^*-\tilde{V}_{W',m}^*$.
\end{enumerate}
These three cases confirms the modularity of $\tilde{V}_{K,m}^*$. With this conclusion, we can easily derive the following two facts:
\begin{enumerate}
    \item As $g^{\text{SoM}}(\vX_K)$ is a convex combination of $\tilde{V}_{K,m}^*$, we have $g^{\text{SoM}}(\vX_U) - g^{\text{SoM}}(\vX_{U'}) \geq g^{\text{SoM}}(\vX_V) - g^{\text{SoM}}(\vX_{V'})$.
    \item Let $z^*=z_1^*=z_2^*=\cdots=z_m^*$ and $m_{U'}^*=\argmax_{1\leq m\leq M}(\tilde{V}_{U',m}^* -z_m^*)=\argmax_{1\leq m\leq M}(\tilde{V}_{W,m}^*-z_m^*)=m_W^*$. We have
    \begin{align}
        (\tilde{V}_{U,m_U^*}^*-z^*)-(\tilde{V}_{U',m_{U'}^*}^*-z^*)
        =(\tilde{V}_{U,m_U^*}^*-z^*)-(\tilde{V}_{U',m_{W}^*}^*-z^*)
        \geq (\tilde{V}_{U,m_{W}^*}^*-z^*)-(\tilde{V}_{U',m_{W}^*}^*-z^*).
    \end{align}
    Since $\tilde{V}_{K,m}^*$ has modularity, we have
    \begin{align}
        (\tilde{V}_{U,m_{W}^*}^*-z^*)-(\tilde{V}_{U',m_{W}^*}^*-z^*)
        \geq (\tilde{V}_{W,m_{W}^*}^*-z^*)-(\tilde{V}_{W',m_{W}^*}^*-z^*)
        \geq (\tilde{V}_{W,m_{W}^*}^*-z^*)-(\tilde{V}_{W',m_{W'}^*}^*-z^*).
    \end{align}
    The above two relations yield
    \begin{align}
        (\tilde{V}_{U,m_U^*}^*-z^*)-(\tilde{V}_{U',m_{U'}^*}^*-z^*)
        \geq (\tilde{V}_{W,m_{W}^*}^*-z^*)-(\tilde{V}_{W',m_{W'}^*}^*-z^*),
    \end{align}
    which shows that $g^{\text{TCH-Set}}(\vX_U) - g^{\text{TCH-Set}}(\vX_{U'}) \geq g^{\text{TCH-Set}}(\vX_V) - g^{\text{TCH-Set}}(\vX_{V'})$.
\end{enumerate}
Then, we demonstrate the modularity of $g_{\mu}^{\text{SSoM}}$ and $g_{\mu}^{\text{STCH-Set}}$. 

First, we have 
\begin{align}
    g_{\mu}^{\text{SSoM}}(\vX_U) - g_{\mu}^{\text{SSoM}}(\vX_{U'}) &= \sum_{m=1}^{M} \mu \log\left( 1+ \frac{e^{-\frac{\tilde{v}_m(\vx')}{\mu}}}{\sum_{k=1}^{U} e^{-\frac{\tilde{v}_m(\vx^{(k)})}{\mu}}} \right), \\
     g_{\mu}^{\text{SSoM}}(\vX_W) - g_{\mu}^{\text{SSoM}}(\vX_{W'}) &= \sum_{m=1}^{M} \mu \log\left( 1+ \frac{e^{-\frac{\tilde{v}_m(\vx')}{\mu}}}{\sum_{k=1}^{W} e^{-\frac{\tilde{v}_m(\vx^{(k)})}{\mu}}} \right).
\end{align}
Since $\sum_{k=1}^{U} e^{-\frac{\tilde{v}_m(\vx^{(k)})}{\mu}}\leq\sum_{k=1}^{W} e^{-\frac{\tilde{v}_m(\vx^{(k)})}{\mu}}$, we can derive $g_{\mu}^{\text{SSoM}}(\vX_U) - g_{\mu}^{\text{SSoM}}(\vX_{U'}) \geq g_{\mu}^{\text{SSoM}}(\vX_W) - g_{\mu}^{\text{SSoM}}(\vX_{W'})$.   
In a similar way, we could prove $g_{\mu}^{\text{STCH-Set}}(\vX_U) - g_{\mu}^{\text{STCH-Set}}(\vX_{U'}) \geq g_{\mu}^{\text{STCH-Set}}(\vX_W) - g_{\mu}^{\text{STCH-Set}}(\vX_{W'})$. Therefore, Theorem~\ref{thm_informal_supermodularity} holds.
\end{proof}

\subsection{Proof of Theorem~\ref{thm_pareto_optimality}}

\textbf{Theorem~\ref{thm_pareto_optimality} }(Pareto Optimality of Solutions). 
\textit{All solutions in the optimal solution set $\vX^*_K$ for SoM~(\ref{eq_som}), SSoM~(\ref{eq_ssom}) or STCH-Set~(\ref{eq_stch_set_scalarization_same_u}) scalarization are Pareto optimal for the original multi-objective optimization problem~(\ref{eq_mop}).}  
\begin{proof}

The Pareto optimality of the solutions found by SoM and STCH-Set have been proved in previous work \cite{lin2025few,liu2025few}. Here, we extend the analysis to SSoM. 

The proof employs contradiction based on the Definition~\ref{def_pareto_optimality} for (weakly) Pareto optimality and the general form of the weighted Smooth Sum-of-Minimum (Weighted SSoM) scalarization:
\begin{equation}
g_\mu^{\text{SSoM}}(\vX_K; \vlambda) = -\mu \sum_{m=1}^{M} \lambda_m \log\left( \sum_{k=1}^{K} e^{-\tilde{v}_m(\vx^{(k)}) / \mu} \right),
\label{eq_supp_ssom_optimality_general}
\end{equation}
where $\vlambda = (\lambda_1, \dots, \lambda_M)$ with $\lambda_m \geq 0$ for all $m$.

\paragraph{Weakly Pareto Optimality}

Let $\vX^*_K = \{\vx^{(1)}, \dots, \vx^{(K)} \}$ be an optimal solution set minimizing $g_\mu^{\text{SSoM}}(\cdot; \vlambda)$. Assume, for the sake of contradiction, that a specific solution $\vx^{(p)} \in \vX^*_K$ is not weakly Pareto optimal for problem (\ref{eq_mop}). By definition, there must exist another solution $\hat{\vx} \in \mathcal{X}$ that strictly dominates it:
\begin{equation}
\tilde{v}_m(\hat{\vx}) < \tilde{v}_m(\vx^{(p)}) \quad \text{for all } m = 1,\dots,M.
\label{eq:strict_domination}
\end{equation}

Now, construct an alternative solution set $\hat{\vX}_K$ by substituting $\hat{\vx}$ for $\vx^{(p)}$ in $\vX^*_K$:
\begin{equation}
\hat{\vX}_K = \{\vx^{(1)}, \dots, \vx^{(p-1)}, \hat{\vx}, \vx^{*(p+1)}, \dots, \vx^{*(K)}\}.
\label{eq_solution_set_x_hat}
\end{equation}

For each objective $m$, inequality (\ref{eq:strict_domination}) implies $e^{-\tilde{v}_m(\hat{\vx})/\mu} > e^{-\tilde{v}_m(\vx^{*(p)})/\mu}$. Since all other solutions remain unchanged, we have:
\begin{equation}
\sum_{k=1}^{K} e^{-\tilde{v}_m(\hat{\vx}^{(k)})/\mu} > \sum_{k=1}^{K} e^{-\tilde{v}_m(\vx^{*(k)})/\mu} \quad \forall m,
\end{equation}
where $\hat{\vx}^{(k)}$ denotes the $k$-th solution in $\hat{\vX}_K$. The logarithm function is strictly increasing, and $\lambda_m \geq 0$, therefore:
\begin{equation}
-\lambda_m \mu \log\left( \sum_{k=1}^{K} e^{-\tilde{v}_m(\hat{\vx}^{(k)})/\mu} \right) < -\lambda_m \mu \log\left( \sum_{k=1}^{K} e^{-\tilde{v}_m(\vx^{*(k)})/\mu} \right) \quad \forall m.
\end{equation}

Summing over all objectives yields:
\begin{equation}
g_{\mu}^{\text{SSoM}}(\hat{\vX}_K; \vlambda) < g_{\mu}^{\text{SSoM}}(\vX^*_K; \vlambda).
\end{equation}
This inequality contradicts the initial assumption that $\vX^*_K$ is a minimizer of $g_\mu^{\text{SSoM}}(\cdot; \vlambda)$. Hence, the assumption was false, and every solution in $\vX^*_K$ must be weakly Pareto optimal for the original multi-objective optimization probelm (\ref{eq_mop}).

\paragraph{Pareto Optimality}

We now establish two sufficient conditions under which the solutions are Pareto optimal, not merely weakly Pareto optimal.

\textbf{Condition 1: Uniqueness of the Optimal Set.}
Suppose $\vX^*_K$ is the unique global minimizer of $g_\mu^{\text{SSoM}}(\cdot; \vlambda)$. Assume a solution $\vx^{(p)} \in \vX^*_K$ is not Pareto optimal, then there exists $\hat{\vx} \in \mathcal{X}$ such that:
\begin{equation}
\tilde{v}_m(\hat{\vx}) \leq \tilde{v}_m(\vx^{(p)}) ; \forall m, \quad \text{and} \quad \tilde{v}_{m'}(\hat{\vx}) < \tilde{v}_{m'}(\vx^{(p)}) \text{ for some } m'.
\label{eq:domination}
\end{equation}

Construct $\hat{\vX}_K$ as in (\ref{eq_solution_set_x_hat}). From (\ref{eq:domination}), it follows that for objective $m'$, the inner sum in $g_\mu^{\text{SSoM}}(\hat{\vX}_K; \vlambda)$ strictly increases when replacing $\vx^{(p)}$ with $\hat{\vx}$, while for other objectives ($m \neq m'$), the sum does not decrease. Since the logarithm is increasing (and hence the negative logarithm is decreasing) and $\lambda_m \geq 0$, we obtain:
\begin{equation}
g_\mu^{\text{SSoM}}(\hat{\vX}_K; \vlambda) \leq g_\mu^{\text{SSoM}}(\vX^*_K; \vlambda).
\label{eq:no_worse}
\end{equation}
However, the uniqueness of $\vX^*_K$ implies that any distinct set $\hat{\vX}_K$ must satisfy $g_\mu^{\text{SSoM}}(\hat{\vX}_K; \vlambda) > g_\mu^{\text{SSoM}}(\vX^*_K; \vlambda)$. This contradicts (\ref{eq:no_worse}). Therefore, under the uniqueness condition, all solutions in $\vX^*_K$ are Pareto optimal.

\textbf{Condition 2: All Positive Weights.}
Now suppose all weights are strictly positive: $\lambda_m > 0$ for $m=1,\dots,M$. Again, assume a solution $\vx^{(p)} \in \vX^*_K$ is not Pareto optimal, so a dominating $\hat{\vx}$ satisfying (\ref{eq:domination}) exists. For the objective $m'$ where strict improvement occurs, the positivity of $\lambda_{m'}$ guarantees:
\begin{equation}
-\lambda_{m'} \mu \log\left( \sum_{k=1}^K e^{-\tilde{v}_{m'}(\hat{\vx}^{(k)})/\mu} \right) < -\lambda_{m'} \mu \log\left( \sum_{k=1}^K e^{-\tilde{v}_{m'}(\vx^{(k)})/\mu} \right).
\end{equation}

For objectives $m \neq m'$, the corresponding terms are less than or equal. Summing over all $M$ objectives, the strict inequality for $m'$ ensures the total sum is strictly smaller:
\begin{equation}
g_\mu^{\text{SSoM}}(\hat{\vX}_K; \vlambda) < g_\mu^{\text{SSoM}}(\vX^*_K; \vlambda).
\end{equation}
This result contradicts the optimality of $\vX^*_K$. Consequently, when all preferences are positive, every solution in the optimal set $\vX^*_K$ is Pareto optimal.

In this work, we use the SSoM with uniform weights for the QD optimization ($\lambda_m = 1$ for all $M$ objectives), which satisfies condition 2 of all positive weights. Therefore, all solutions in the optimal solution set $\vX^*_K$ for SSoM scalarization (\ref{eq_ssom}) are Pareto optimal for the original multi-objective optimization problem~(\ref{eq_mop}). 

\end{proof}

\subsection{Proof of Theorem~\ref{thm_tch_set_pareto_optimality}}

\textbf{Theorem~\ref{thm_tch_set_pareto_optimality} }(Existence of Pareto Optimal Solution). 
\textit{For the TCH-Set scalarization (\ref{eq_tch_set_scalarization}), there exists at least one optimal solution set $\vX^*_K$ whose solutions are all Pareto optimal for (\ref{eq_mop}). Moreover, if the optimal set $\vX^*_K$ is unique, then all its solutions are Pareto optimal.}  
\begin{proof}

The proof of Theorem~\ref{thm_tch_set_pareto_optimality} can be found in \cite{lin2025few,liu2025few}.

\end{proof}

\subsection{Proof of Theorem~\ref{thm_bounded_approximation}}

\textbf{Theorem~\ref{thm_bounded_approximation} }(Smooth Approximation). 
\textit{SSoM scalarization (\ref{eq_ssom}) is a uniform smooth approximation of SoM scalarization (\ref{eq_som}). STCH-Set scalarization (\ref{eq_stch_set_scalarization_same_u}) is a uniform smooth approximation of TCH-Set scalarization (\ref{eq_tch_set_scalarization}).}  
\begin{proof}

The previous work~\cite{lin2025few} has proved that the STCH-Set scalarization (\ref{eq_stch_set_scalarization_same_u}) is a uniform smooth approximation of TCH-Set scalarization (\ref{eq_tch_set_scalarization}). Here, we extend the analysis to SSoM. 

Formally, we prove that the Smooth Sum-of-Minimum (SSoM) scalarization is a uniform smooth approximation of the (non-smooth) Sum-of-Minimum (SoM) scalarization:
\begin{align}
&g^{\text{SoM}}(\vX_K; \vlambda) = \sum_{m=1}^{M} \lambda_m \min_{\vx \in \vX_K} \tilde{v}_m(\vx), \\
&g_\mu^{\text{SSoM}}(\vX_K; \vlambda) = -\mu \sum_{m=1}^{M} \lambda_m \log\left( \sum_{k=1}^{K} e^{-\tilde{v}_m(\vx^{(k)}) / \mu} \right),
\end{align}
where $\vlambda = (\lambda_1, \dots, \lambda_M)$ with $\lambda_m \geq 0$ for all $m$, and $\mu > 0$ is a smooth parameter. 

The proof follows the standard technique of establishing upper and lower bounds for the approximation and showing their tight convergence \cite{nesterov2005smooth}.

\paragraph{Bounds for the Smooth Minimum Operator}

The core of the approximation lies in the smooth minimum operator $\smin_\mu$. For any set of values $\{a_k\}_{k=1}^K = \{a_1,\dots,a_K \} $, it is known \cite{bertsekas2003convex, boyd2004convex} that the following inequalities hold:
\begin{equation}
\smin_\mu(\{a_k\}_{k=1}^K) \leq \min(\{a_k\}_{k=1}^K) \leq \smin_\mu(\{a_k\}_{k=1}^K) + \mu \log K,
\label{eq:smin_bounds_general}
\end{equation}
where $\smin_\mu(\{a_k\}_{k=1}^K) = -\mu \log \left( \sum_{k=1}^K e^{-a_k / \mu} \right)$. As $\mu \downarrow 0$, both bounds converge to $\min(\{a_k\}_{k=1}^K)$, making the approximation tight.

\paragraph{Bounding the SoM Scalarization}

Applying the bounds (\ref{eq:smin_bounds_general}) to each term in the SoM scalarization, we obtain for each objective $m$:
\begin{equation}
-\mu \log\left( \sum_{k=1}^{K} e^{-\tilde{v}_m(\vx^{(k)}) / \mu} \right) \leq \min_{\vx \in \vX_K} \tilde{v}_m(\vx) \leq -\mu \log\left( \sum_{k=1}^{K} e^{-\tilde{v}_m(\vx^{(k)}) / \mu} \right) + \mu \log K.
\end{equation}

Multiplying each inequality by $\lambda_m \geq 0$ and summing over all $M$ objectives yields the desired bounds for the full scalarization:
\begin{equation}
g_\mu^{\text{SSoM}}(\vX_K; \vlambda) \leq 
g^{\text{SoM}}(\vX_K; \vlambda) \leq g_\mu^{\text{SSoM}}(\vX_K; \vlambda) + \mu \sum_{m=1}^{M} \lambda_m \log K. 
\label{eq_ssom_bounds}
\end{equation}

\paragraph{Uniform Convergence}

From inequalities (\ref{eq_ssom_bounds}), it is evident that for any fixed solution set $\vX_K$ and weight vector $\vlambda$, the approximation error is bounded by:
\begin{equation}
0 \leq g^{\text{SoM}}(\vX_K; \vlambda) - g_\mu^{\text{SSoM}}(\vX_K; \vlambda) \leq  \mu \sum_{m=1}^{M} \lambda_m \log K.
\end{equation}
This error bound is linear in $\mu$ and independent of $\vX_K$ for fixed $K$ and $\vlambda$. Therefore, as the smoothing parameter $\mu$ approaches zero, the bound tightens uniformly:
\begin{equation}
\lim_{\mu \downarrow 0} g_\mu^{\text{SSoM}}(\vX_K; \vlambda) = g^{\text{SoM}}(\vX_K; \vlambda) \quad \text{for all } \vX_K \subset \mathcal{X}.
\end{equation}
The existence of such uniform error bounds that vanish with $\mu$ is the defining characteristic of a uniform smooth approximation \cite{nesterov2005smooth}. Hence, $g_\mu^{\text{SSoM}}(\cdot; \vlambda)$ is a uniform smooth approximation of $g^{\text{SoM}}(\cdot; \vlambda)$.

\end{proof}

\clearpage
\section{Derivation of STCH-Set}
\label{supp_sec_ssom_stch-set}

The Tchebycheff Set scalarization (TCH-Set) (\ref{eq_tch_set_scalarization}) involves non-differentiable $\max$ and $\min$ operators, which hinder efficient gradient-based optimization even when the underlying objective functions are smooth. To address this, \citet{lin2025few} adopt the standard log-sum-exp smoothing technique \cite{beck2012smoothing} to develop the smooth Tchebycheff Set scalarization (STCH-Set) for multi-objective optimization. 

For a set of values $\{a_i\}_{i=1}^n$, the smooth maximum and smooth minimum approximations with parameter $\mu > 0$ are defined as:
\begin{align}
&\smax_{\mu}(a_1, \dots, a_n) = \mu \log\left( \sum_{i=1}^n e^{a_i / \mu} \right), \\
&\smin_{\mu}(a_1, \dots, a_n) = -\mu \log\left( \sum_{i=1}^n e^{-a_i / \mu} \right).
\end{align}
As $\mu \to 0$, these operators recover the standard $\max$ and $\min$.

The STCH-Set scalarization applies this smoothing to the TCH-Set scalarization $g^{\text{TCH-Set}}(\vX_K; \vlambda)$ from (\ref{eq_tch_set_scalarization}). The outer $\max$  and the inner $\min$ are replaced by their smooth counterparts $\smax_{\mu}$ and $\smin_{\mu_m}$, respectively, with separate smoothing parameters $\mu$ (for the $\max$) and $\{\mu_m \}_{m=1}^M$ (for the $\min$ ). For generality, we also include a weight vector $\vlambda$ and a reference point $\vz^*$. This yields the general smooth formulation:
\begin{align}
g_{\mu, {\mu_m}}^{\text{(STCH-Set)}}(\vX_K; \vlambda) &= \smax_{\mu} \left\{ \lambda_m \left( \smin_{\mu_m} \left\{ \tilde{v}_m(\vx^{(k)}) \right\}_{k=1}^K - z_m^* \right) \right\}_{m=1}^M \nonumber \\
&= \mu \log \left( \sum_{m=1}^{M} \exp\Big( \frac{ \lambda_m \big( \smin_{\mu_m} { \tilde{v}_m(\vx^{(k)}) } - z_m^* \big) }{\mu} \Big) \right).
\label{eq_stch_set_general}
\end{align}

Substituting the definition of $\smin_{\mu_m}$ into the above equation gives the full expression:
\begin{align}
g_{\mu, {\mu_m}}^{\text{(STCH-Set)}}(\vX_K; \vlambda) &= \mu \log \Bigg( \sum_{m=1}^{M} \exp\Bigg( \frac{ \lambda_m \Big( -\mu_m \log\big( \sum_{k=1}^{K} e^{ -\tilde{v}_m(\vx^{(k)}) / \mu_m } \big) - z_m^* \Big) }{\mu} \Bigg) \Bigg). 
\label{eq_stch_set_full}
\end{align}

In this work, following \citet{lin2025few}, we use a single shared smoothing parameter $\mu$ for all operators (i.e., setting $\mu_m = \mu$ for all $m$) to simplify the smooth approximation. Applying this to the full expression~(\ref{eq_stch_set_full}) and assuming uniform preferences ($\lambda_m = 1$ for all $m$) as in the standard QD setting, we obtain the simplified STCH-Set scalarization~(\ref{eq_stch_set_scalarization_same_u}):
\begin{equation}
g_{\mu}^{\text{STCH-Set}}(\vX_K) = \mu \log\left( \sum_{m=1}^{M} e^{ - \log\left( \sum_{k=1}^{K} e^{- \frac{\tilde{v}_m(\vx^{(k)})}{\mu}} \right) - z_m^* } \right).
\end{equation}
This formulation is fully differentiable with respect to the solutions $\vx^{(k)}$, enabling efficient gradient-based optimization of the entire solution set $\vX_K$.

\clearpage
\section{Problem Details}
\label{supp_sec_problem}

\subsection{Linear Projection}

\begin{figure*}[h]
\centering
\subfloat[2D Heatmap]{\includegraphics[width = 0.3\linewidth]{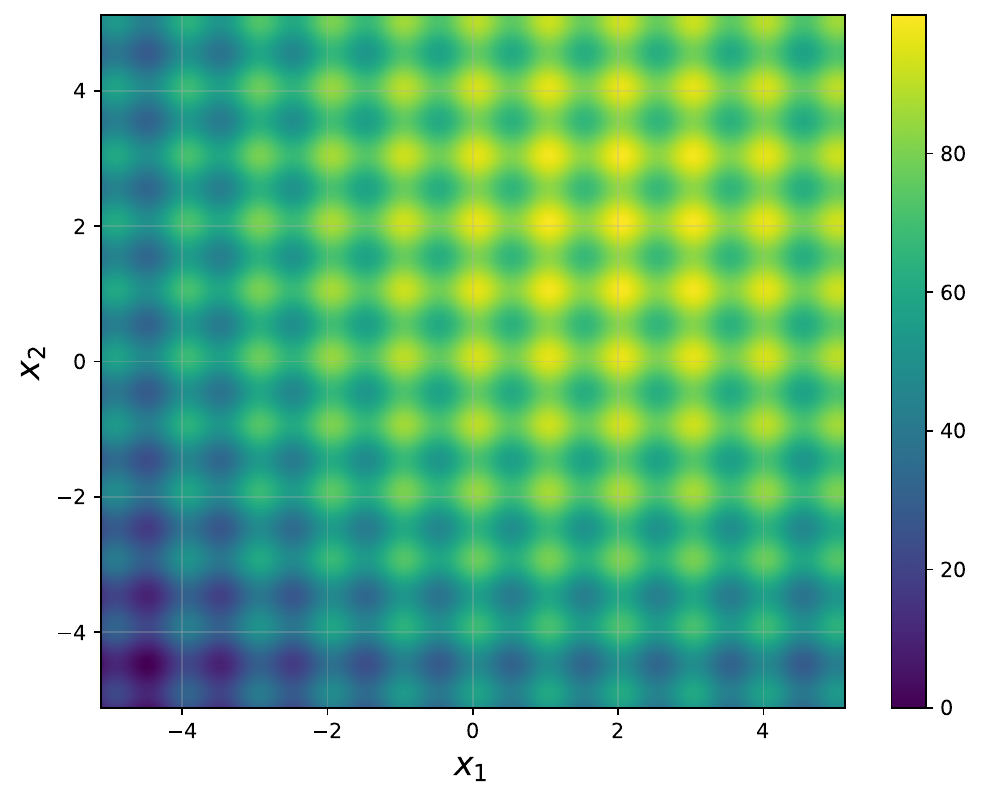}}
\subfloat[3D Surface Plot]{\includegraphics[width = 0.3\linewidth]{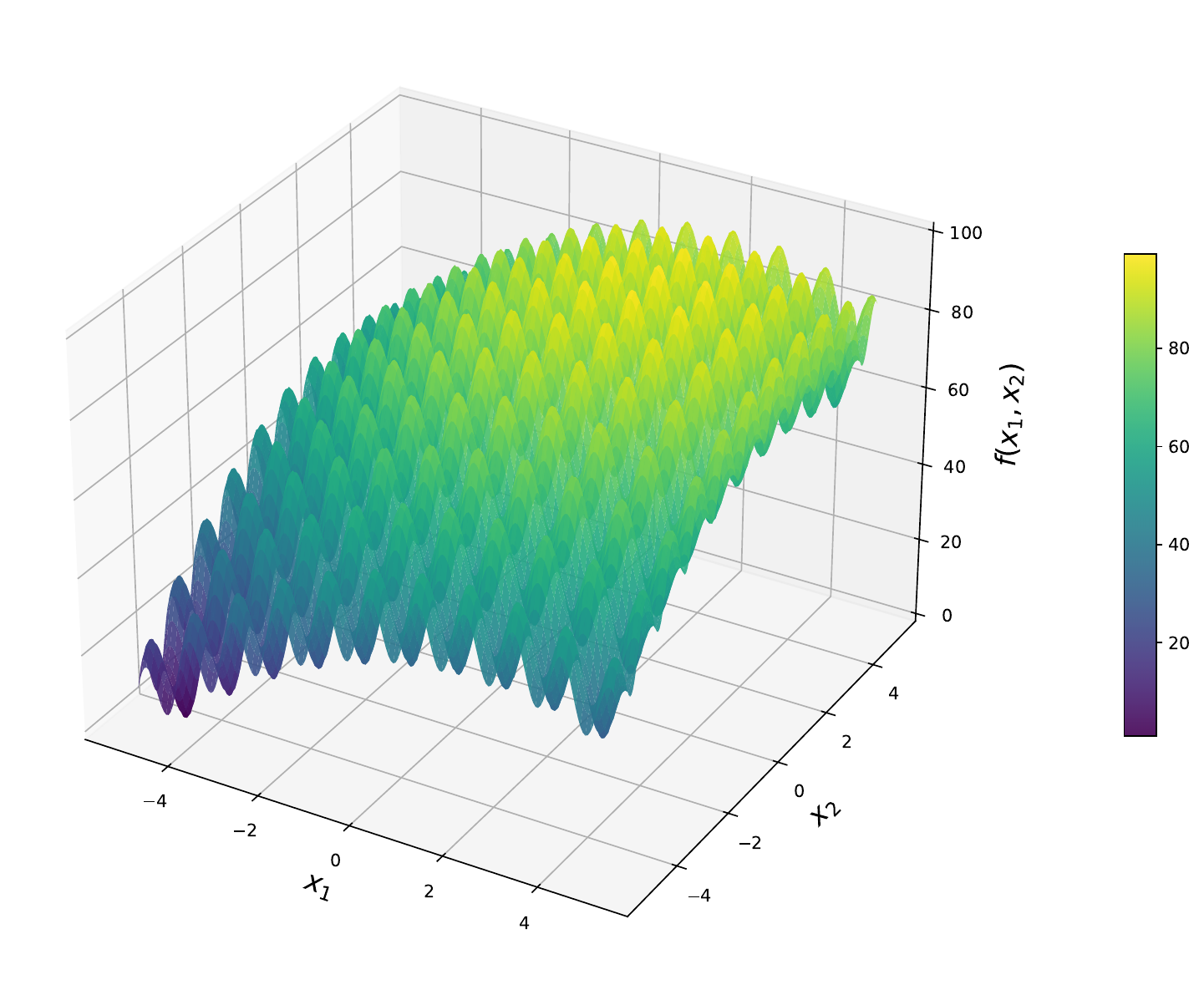}}\hfill
\caption{Rastrigin Function}
\label{fig_rastrigin}
\end{figure*}

Linear Projection (LP) is a benchmark designed to evaluate the scalability of quality-diversity (QD) algorithms in high-dimensional behavior space~\cite{fontaine2020covariance}. It is built upon the multimodal Rastrigin function~\cite{rastrigin1974systems,hoffmeister1990genetic}, coupled with a controllable projection into a behavioral space of predefined dimensionality.

Following the setting in \citet{hedayatian2025soft}, this problem is formulated as follows:

\begin{enumerate}

    \item \textbf{Solution Space:} Algorithms optimize a $1024$-dimensional real-valued vector $\mathbf{x} \in \mathbb{R}^{1024}$, with each variable bounded within $[-5.12, 5.12]$.

    \item \textbf{Objective Function:} The goal is to maximize a transformed version of the Rastrigin function, which is defined as:
    \begin{equation}
    f_{\text{Rastrigin}}(\vx) = 10n + \sum_{i=1}^{n} \left[ x_i^2 - 10 \cos(2\pi x_i) \right],
    \end{equation}
    where $n=1024$. To increase difficulty, its global optimum is shifted from the origin to $[2.048, \dots, 2.048]^\top \in \mathbb{R}^{1024}$. The problem is then converted to maximization via a linear transformation that normalizes the objective to the range $[0, 100]$:
    \begin{equation}
    f(\mathbf{x}) = 100 \times \frac{ M_{\text{Rastrigin}} - f_{\text{Rastrigin}}(\mathbf{x})}{M_{\text{Rastrigin}}},
    \end{equation}
    where $M_{\text{Rastrigin}}$ is the maximum value of the Rastrigin function in the search domain. The 2D and 3D plots of the transformed Rastrigin function can be found in Figure~\ref{fig_rastrigin}. 

    \item \textbf{Behavior Descriptor:} To define a $d$-dimensional behavior space, the $1024$-dimensional solution vector $\vx$ is partitioned into $d$ equally sized blocks. For each block, the mean of its components is computed and clipped to $[-5.12, 5.12]$. Formally, the $k$-th component of the behavior descriptor is:
    \begin{equation}
    \vb_k(\vx) = \frac{1}{d} \sum_{i=\frac{(k-1)n}{d} + 1}^{\frac{kn}{d}} \mathrm{clip}(x_i), \quad k = 1, \dots, d,
    \end{equation}
    where the clipping function is defined as:
    \begin{equation}
    \mathrm{clip}(x_i) = \begin{cases} 
        x_i & \text{if } -5.12 \leq x_i \leq 5.12, \\
        \frac{5.12}{x_i} & \text{otherwise}.
    \end{cases}
    \end{equation}
    By varying $d$, we can systematically adjust the dimensionality of the behavioral space. In this work, we consider the cases $d \in \{ 4, 8, 16 \}$.
\end{enumerate}

Algorithms must simultaneously achieve high objective values while ensuring their solutions are diverse in this projected $d$-dimensional behavior space. As $d$ increases, the challenge of maintaining behavioral diversity grows substantially, allowing a rigorous assessment of how different QD methods scale with the dimensionality of the behavior space. 

\subsection{Image Composition}

\begin{figure*}[h]
    \centering
    \includegraphics[width= 0.25 \linewidth]{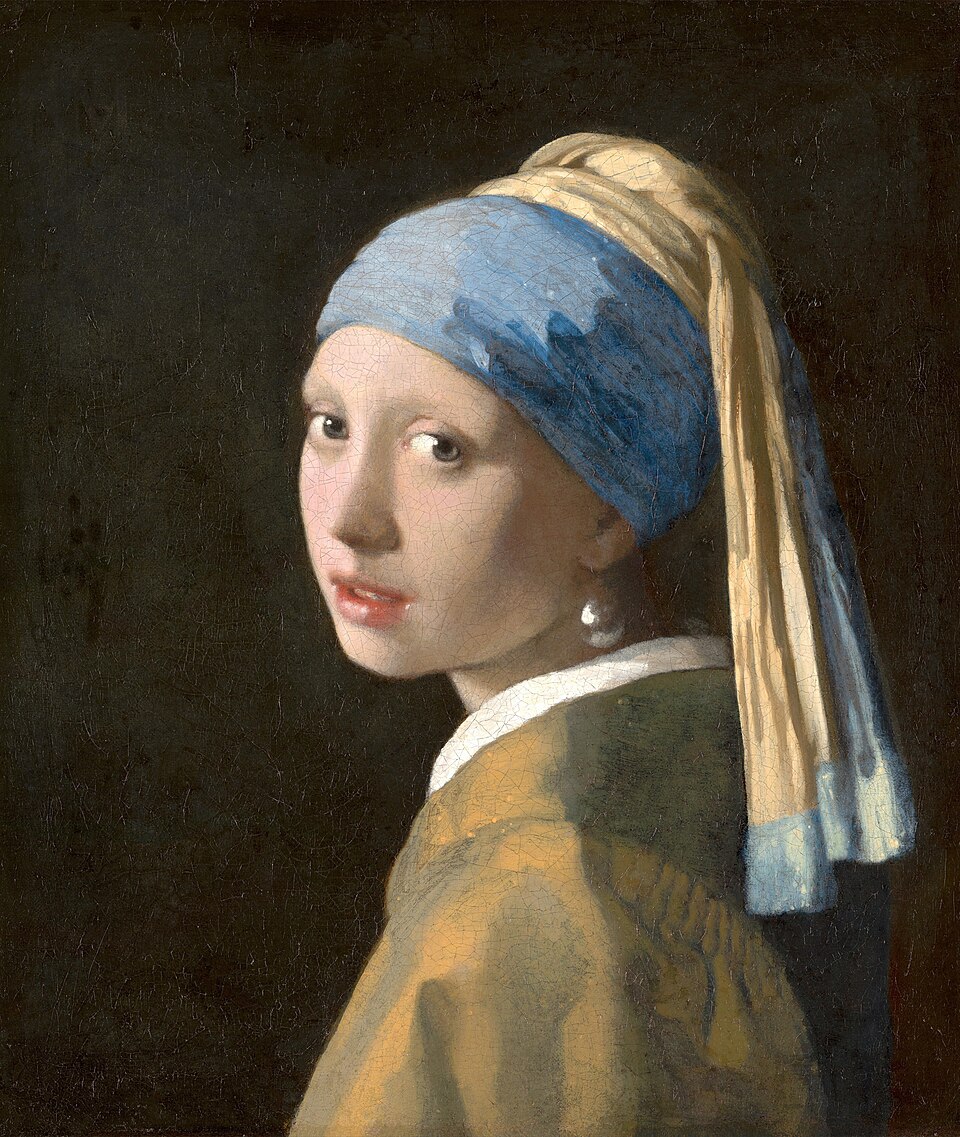}
    \caption{Girl with a Pearl Earring}
    \label{fig_girl_pearl_earing}
\end{figure*}

Image Composition (IC) is a new differentiable quality-diversity benchmark recently introduced in \citet{hedayatian2025soft}, which is inspired by computational creativity applications~\cite{tian2022modern}. It requires QD algorithms to reconstruct a target image by composing multiple parameterized circles while exploring diverse visual styles in a continuous behavior space. The benchmark is constructed as follows:

\begin{enumerate}
    \item \textbf{Solution Space:} Each solution is represented by a parameter vector that defines a set of $1024$ overlapping circles. Each circle is parameterized by 7 values: the $(x, y)$ center coordinates, radius, three RGB color channels, and an opacity coefficient. These parameters are initially represented as unconstrained logits, which are passed through sigmoidal transformations and rescaled to their respective valid ranges before rendering. Therefore, the solution dimension is $n = 1024 \times 7 = 7168$.

    \item \textbf{Objective Function:} The quality of a solution is determined by its similarity to a fixed target image. A fully differentiable renderer composites the parameterized circles, with soft edges controlled by a smoothness hyperparameter, onto a $64 \times 64$ black canvas using alpha-compositing. The similarity between the rendered image and the target is measured using the Structural Similarity Index Measure (SSIM)~\cite{wang2024image}, which is subsequently normalized to the range $[0, 100]$. Following \citet{hedayatian2025soft}.

    \item \textbf{Behavior Descriptor:} A 5-dimensional behavior space is defined to capture diverse visual styles:
    \begin{itemize}
        \item \textbf{Mean Radius of Circles},
        
        \item \textbf{Variance of Radii},
        
        \item \textbf{Color Spread}: variety of RGB values in the palette,
        
        \item \textbf{Color Harmony}: coherence of circle hues in HSV space,

        \item \textbf{Degree of Spatial Clustering} based on average 5-nearest-neighbor distances.
    \end{itemize}
    Each descriptor is normalized to the range $[0, 1]$, where higher values correspond to larger mean radius, larger variance of radii, greater diversity, more harmony, or tighter clustering, respectively.

\end{enumerate}

The entire evaluation pipeline, from solution parameters to the final objective score and behavior descriptor vector, is fully differentiable. This property enables the direct application of gradient-based optimizers throughout the quality-diversity search process. Following \citet{hedayatian2025soft}, in our experiments, the target image is Johannes Vermeer's painting ``Girl with a Pearl Earring'' from Wikimedia Commons as shown in Figure~\ref{fig_girl_pearl_earing}.

\subsection{Latent Space Illumination}

Latent Space Illumination (LSI)~\cite{fontaine2021differentiable} is a challenging quality-diversity benchmark that requires algorithms to explore the high-dimensional latent space of a pre-trained generative model, such as StyleGAN2~\cite{karras2020analyzing}. The core challenge lies in discovering a diverse set of high-quality latent vectors that, when decoded by the generator, produce images that are both visually faithful to a target concept and varied along specified semantic attributes. 

Following the setup of previous work~\cite{ hedayatian2025soft}, this benchmark is constructed as follows with two variants (base version and hard version):

\begin{enumerate}
\item \textbf{Solution Space:} Algorithms search within the latent space of a pre-trained StyleGAN2~\cite{karras2020analyzing}. Each solution corresponds to a point in this $9216$-dimensional latent space, which serves as input to the generator to produce an image.

\item \textbf{Objective Function:} The quality of a solution is determined by how well the generated image matches a specified target concept. In this work, we adopt the standard LSI formulation~\cite{fontaine2023covariance, hedayatian2025soft} where the objective is to generate high-quality images of a specific person (e.g., “A photo of Tom Cruise” in base version) or a thematic concept (e.g., “A photo of a detective from a noir film” in hard version). The quality score is computed by first generating an image from the latent vector using StyleGAN2, then comparing it to the target concept using CLIP~\cite{radford2021learning}, which quantifies the semantic similarity between the generated image and a text prompt describing the target. 

\item \textbf{Behavior Descriptor:} The behavior diversity of a population of generated images is quantified along a set of interpretable, pre-specified attributes. These attributes are computed by using CLIP to compare each generated image against two descriptive text sentences, of which one that positively represents the attribute and one that negatively represents it.

\paragraph{Base Version} 

In the base version, the behavior space is 2-dimensional, capturing attributes such as age and hair length:

\begin{itemize}
    \item (“Photo of Tom Cruise as a small child”, “Photo of Tom Cruise as an elderly person”)

    \item (“Photo of Tom Cruise with long hair”, “Photo of Tom Cruise with short hair”)
\end{itemize}

\paragraph{Hard Version}

To create a more challenging optimization landscape, QD algorithms are also evaluated on a hard version, which extends the behavior space to 7 dimensions:

\begin{itemize}
    \item (“Photo of a young kid”, “Photo of an elderly person”)

    \item (“Photo of a person with long hair”, “Photo of a person with short hair”)

    \item (“Photo of a person with dark hair”, “Photo of a person with white hair”)

    \item (“Photo of a person smiling”, “Photo of a person frowning”)

    \item (“Photo of a person with a round face”, “Photo of a person with an oval face”)

    \item (“Photo of a person with thin, sparse hair”, “Photo of a person with thick, full hair”)

    \item (“Photo of a person looking directly into the camera”, “Photo of a person looking
sideways”)
    
\end{itemize}

\end{enumerate}

\clearpage
\section{Experiment Setting}
\label{supp_sec_setting}

\subsection{Metrics}

To comprehensively evaluate the performance of QD algorithms, we employ a set of metrics that assess solution quality, population diversity, and their combined overall QD performance.

\subsection*{1. Quality Metrics}
These metrics measure the quality of the found solutions.
\begin{itemize}
    \item \textbf{Mean Objective:} The mean of the objective values $f_i$ for all solutions in the solution set. 
    
    \item \textbf{Max Objective:} The highest objective value found in the solution set.
    
\end{itemize}

\subsection*{2. Diversity Metrics}
These metrics quantify the behavioral diversity of solutions.
\begin{itemize}
    \item \textbf{Coverage:} This metric relies on a discretization of the behavior space using a fixed-size CVT archive~\cite{vassiliades2017using}. Following \citet{hedayatian2025soft}, we use $512$ cells for LP and LSI and $1024$ cells for IC. Coverage is defined as the percentage of cells in this archive that are occupied by at least one solution. It directly measures the breadth of exploration in a discretized behavioral grid.
    
    \item \textbf{Vendi Score (VS)~\cite{friedman2023vendi}:} A discretization-free measure of diversity that calculates the effective number of unique behaviors in a solution set. Given a set of solutions with behavior descriptors $\{\mathbf{b}_i\}_{i=1}^n$, we construct a similarity matrix $\mathbf{K} \in \mathbb{R}^{n\times n}$ using a Gaussian kernel:
    \begin{equation}
    \mathbf{K}_{ij} = \exp\left( -\frac{\| \mathbf{b}_i - \mathbf{b}_j \|^2}{\sigma_v^2} \right),
    \end{equation}
    where $\sigma_v^2$ is set to $d/6$ for a $d$-dimensional behavior space as in \citet{hedayatian2025soft}. This is a heuristic to ensure a similarity of $e^{-1}$ for two uniformly random vectors in the behavior space. 
    The Vendi Score is then defined as:
    \begin{equation}
    \mathrm{VS}(\mathbf{K}) = \exp\left( -\sum_{i=1}^{n} \lambda_i \log \lambda_i \right),
    \end{equation}
    where $\lambda_i$ are the eigenvalues of the normalized matrix $\frac{1}{n}\mathbf{K}$. A higher VS indicates greater diversity.
\end{itemize}

\subsection*{3. Overall Performance Metrics}
These metrics jointly consider both quality and diversity for QD optimization.
\begin{itemize}
    \item \textbf{QD Score~\cite{pugh2016quality}:} The sum of the objective values of the best-performing solution in each occupied cell of a fixed CVT archive. It is the traditional performance measure in QD optimization, rewarding both high quality and broad coverage.
    
    \item \textbf{Quality-weighted Vendi Score (QVS)~\cite{nguyen2024quality}:} An extension of the Vendi Score that incorporates solution quality. It is computed by multiplying the Vendi Score by the mean objective of the solution set:
    \begin{equation}
    \mathrm{QVS}(\mathbf{K}, \{f_k\}_{k=1}^{K}) = \mathrm{VS}(\mathbf{K}) \times \frac{1}{n}\sum_{k=1}^{K} f_k.
    \end{equation}
    This metric emphasizes that a high-performing population must exhibit both high diversity and high quality. If the mean objective is negative, we report QVS as $0.0$.
\end{itemize}

\clearpage
\subsection{Hyperparameters}

We follow the hyperparameter configurations as reported in \citet{hedayatian2025soft} and the Soft QD codebase.

\paragraph{Baseline Algorithms}

All MAP-Elites based QD methods use CVT archives for discretizing the behavior space, with $10^4$ cells for the Linear Projection (LP) and Image Composition (IC) benchmarks, and a finer archive of $4 \times 10^4$ cells for the Latent Space Illumination (LSI) benchmark. For LP and IC, a grid search is conducted for each algorithm's most critical hyperparameters and selected the configuration yielding the highest QD Score. Key tuned parameters include:
\begin{itemize}
    \item \textbf{CMA-MEGA:} initial evolution strategy step size ($\sigma_0$) and the optimizer learning rate.
    \item \textbf{CMA-MAEGA:} initial evolution strategy step size ($\sigma_0$), optimizer learning rate, and archive learning rate.
    \item \textbf{Sep-CMA-MAE:} initial evolution strategy step size ($\sigma_0$) and archive learning rate.
    \item \textbf{GA-ME:} iso and line sigma parameters, the gradient step size.
    \item \textbf{DNS:} iso and line sigma parameters, number of nearest neighbors, and learning rate (for the DNS-G variant).
\end{itemize}
For LSI,the default hyperparameters follow the setting in \citet{fontaine2023covariance} that used in the \texttt{pyribs} implementation \cite{tjanaka2023pyribs}. The only modification is the reduced batch size of 16 (instead of 32) due to computational constraints.

\paragraph{SQUAD and MOO-based Approach}

All our introduced MOO-based approaches (SoM, TCH-Set, SSoM, and STCH-Set) follow the same hyperparameter setting for SQUAD~\cite{hedayatian2025soft}. These algorithms uses Adam~\cite{kingma2015adam} as their optimizer, and the default hyperparameter values are listed in Table~\ref{table_squad_params}, where the MOO-based methods do not need the number of neighbors.

\begin{table}[htbp]
\centering
\caption{Default hyperparameters for SQUAD and MOO-based Methods.}
\label{table_squad_params}
\begin{tabular}{lc}
\toprule
\textbf{Parameter} & \textbf{Default Value} \\
\midrule
Population Size  & 1024 \\
Batch Size  & 64 \\
Number of Neighbors  & 16 \\
Learning Rate & 0.05 \\
\bottomrule
\end{tabular}
\end{table}

All the MOO-based approaches and SQUAD are trained for $1000$ iterations in LP and IC, and for $175$ iterations in LSI. The number of iterations for all baseline algorithms are set such that they use at least as many evaluations as the MOO-based approaches and SQUAD in all domains.

\clearpage
\section{Additional Experiment Results and Analysis}
\label{supp_sec_experiment}

\subsection{Effect of Bandwidth $\gamma^2$}

\begin{figure*}[h]
    \centering
    \includegraphics[width= 0.5 \linewidth]{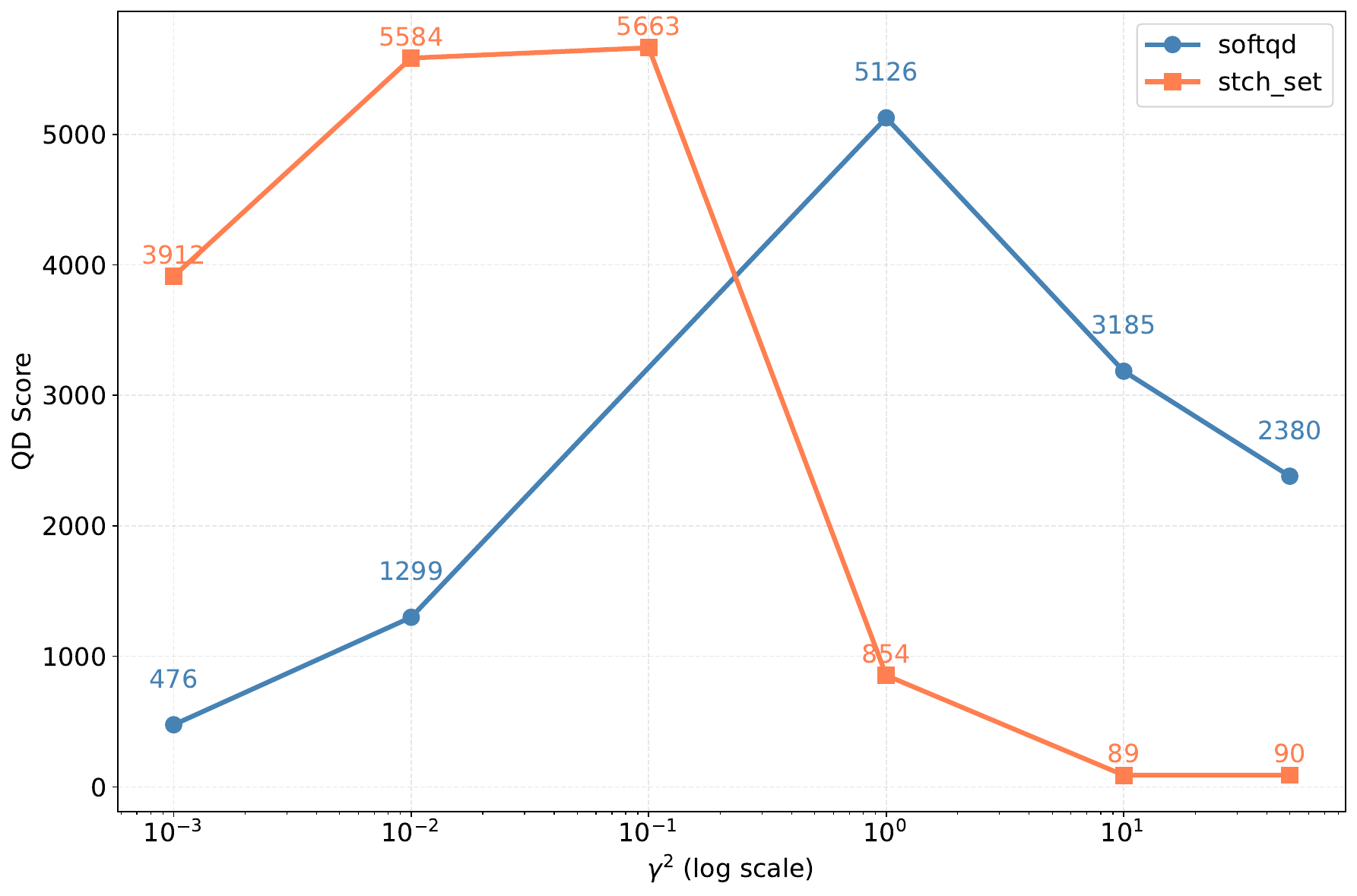}
    \caption{Effect of Bandwidth $\gamma^2$}
    \label{fig_qd_score_gamma}
\end{figure*}

We compare the sensitivity of STCH-Set and the Soft QD baseline SQUAD to the bandwidth parameter $\gamma^2$ on the Image Composition (IC) benchmark. The results reveal a strikingly different response pattern between these two methods as shown in Figure~\ref{fig_qd_score_gamma}. STCH-Set achieves its best performance at a much smaller $\gamma^2=0.1$, and maintains consistently high performance across the range $\gamma^2 \in [0.01, 0.1]$. Moreover, STCH-Set's performance across this effective range is competitive with or superior to SQUAD's best performance. However, STCH-Set's performance degrades rapidly for $\gamma^2 \geq 1$, falling significantly below SQUAD's. We note that attempts to run SQUAD at $\gamma^2 = 0.1$ with three different random seeds all resulted in numerical instability (NaN), a behavior we plan to investigate further in future work.

\subsection{Effect of Solution Set Size $K$}

\begin{table}[h]
\centering
\caption{Effect of Solution Set Size $K$.}
\label{table_population_size}
\begin{tabular}{l c c}
\hline
Algorithm & QD Score (\(\times 10^3\)) & QVS \\
\hline
STCH-Set (\(K = 128\)) & $9.31 \pm 0.16 $  & $333.50 \pm 1.95$  \\
STCH-Set (\(K = 256\)) & $17.42 \pm 0.22 $  & $384.75 \pm 2.63$ \\
STCH-Set (\(K = 512\)) & $31.52 \pm 0.59 $  & $438.70 \pm 3.27$  \\
STCH-Set (\(K = 1024\)) & $52.16 \pm 0.70$ & $584.00 \pm 2.14$ \\
\hline
\end{tabular}
\end{table}

We investigate the effect of the solution set size $K$ for STCH-Set on the Linear Projection (Hard, $d=16$) benchmark. Each configuration was run with $5$ independent random seeds, and the reported values are the mean and standard deviation. As shown in Table~\ref{table_population_size}, both the QD Score and QVS of the STCH-Set algorithm exhibit substantial and consistent improvement as $K$ increases from $128$ to $1024$. The low standard deviations across runs confirm the robustness of this trend. These results demonstrate that our set-based formulation effectively leverages a larger population size to discover a more comprehensive and higher-performing collection of solutions, thereby achieving better coverage of the behavior space.

\subsection{Effect of Smooth Parameter $\mu$}

\begin{table}[h]
\centering
\caption{Effect of Smooth Parameter $\mu$}
\label{table_smooth_parameter}
\begin{tabular}{l c c}
\hline
Algorithm & QD Score (\(\times 10^3\)) & QVS \\ 
\hline
STCH-Set (\(\mu = 0.001\)) & $51.85 \pm 0.63$ & $593.34 \pm 2.17$ \\ 
STCH-Set (\(\mu = 0.01\))  & $52.41 \pm 0.55$ & $584.37 \pm 3.56$ \\ 
STCH-Set (\(\mu = 0.1\))  & $47.17 \pm 0.71$ & $499.20 \pm 1.88$ \\ 
STCH-Set (\(\mu = 0.5\))   & $36.22 \pm 0.48$ & $292.77 \pm 3.24$ \\ 
STCH-Set (\(\mu = 1.0\))   & $35.85 \pm 0.77$ & $223.67 \pm 2.93$ \\ 
\hline
\end{tabular}
\end{table}

We analyze the sensitivity of the STCH-Set algorithm to its smooth parameter $\mu$ on the Linear Projection (Hard, $d=16$) task. As shown in Table~\ref{table_smooth_parameter}, performance remains high and stable for smaller values of $\mu$ ($0.001$ and $0.01$), achieving the best QD Score and QVS. This result aligns with the theoretical property of STCH-Set as a uniform smooth approximation of TCH-Set, which becomes exact as $\mu \rightarrow 0$. As $\mu$ increases to $0.1$ and beyond, both metrics decline noticeably, indicating that excessive smoothing weakens the alignment with the original optimization objective and degrades solution quality and diversity. These results suggest that in practice, a sufficiently small $\mu$ (e.g., $0.001$–$0.01$) should be chosen to balance numerical stability with faithful optimization.

\subsection{Full Results for Linear Projection}

\begin{table*}[ht]
\centering
\setlength{\tabcolsep}{4pt}
\caption{\textbf{Full Results for the Linear Projection (LP) Benchmark.}}
\label{table_lp}
\begin{tabular}{lccccccc}
\toprule
& \multicolumn{2}{c}{\textbf{Quality}} & \multicolumn{2}{c}{\textbf{Diversity}} & \multicolumn{2}{c}{\textbf{Overall Performance}} \\
\cmidrule(lr){2-3} \cmidrule(lr){4-5} \cmidrule(lr){6-7}
\textbf{Algorithm} & \textbf{Mean Objective} & \textbf{Max Objective} & \textbf{Coverage} & \textbf{Vendi Score} & \textbf{QD Score} & \textbf{QVS} \\
\midrule

\multicolumn{7}{c}{\textbf{Easy $(d=4)$}} \\
\midrule
\multicolumn{7}{c}{\textbf{MOO Methods}} \\
SoM                & $60.33 \pm 0.32$        & $89.88 \pm 1.02$       & $60.74 \pm 0.39$   & $5.79 \pm 0.07$      & $38844.5 \pm 498.2$ & $349.46 \pm 5.66$ \\
TCH-Set            & $59.74 \pm 0.03$        & $62.77 \pm 0.39$       & $0.17 \pm 0.00$   & $1.03 \pm 0.00$      & $103.98 \pm 28.58$ & $61.61 \pm 0.06$ \\
SSoM               & $64.67 \pm 0.17$        & $89.24 \pm 0.78$       & $66.68 \pm 0.24$   & $6.65 \pm 0.05$      & $44560.3 \pm 721.3$ & $430.12 \pm 3.31$ \\
STCH-Set           & $65.14 \pm 0.34$        & $89.96 \pm 1.43$       & $66.14 \pm 0.37$   & $6.58 \pm 0.07$      & $44647.7 \pm 486.8$ & $428.34 \pm 6.40$ \\
\addlinespace
\cmidrule(r){1-7}
\multicolumn{7}{c}{\textbf{QD Methods}} \\

SQUAD              & $68.36 \pm 0.05$        & $89.28 \pm 0.47$       & $63.96 \pm 0.09$   & $6.55 \pm 0.02$      & $45442.7 \pm 699.6$ & $448.09 \pm 1.65$ \\
CMA-MAEGA          & $66.05 \pm 1.06$        & $91.54 \pm 1.81$       & $94.81 \pm 0.18$   & $6.94 \pm 0.22$      & $66666.5 \pm 1850.7$ & $458.39 \pm 16.92$ \\
CMA-MEGA           & $66.40 \pm 0.67$        & $94.45 \pm 1.08$       & $98.04 \pm 0.10$   & $7.57 \pm 0.13$      & $71062.8 \pm 1227.8$ & $502.71 \pm 10.62$ \\
CMA-MAE            & $69.81 \pm 0.89$        & $78.71 \pm 0.91$       & $1.09 \pm 0.18$   & $1.25 \pm 0.02$      & $790.1 \pm 119.5$ & $87.20 \pm 1.30$ \\
PGA-ME             & $70.96 \pm 2.35$        & $79.54 \pm 0.80$       & $0.21 \pm 0.10$   & $1.07 \pm 0.01$      & $157.4 \pm 80.0$ & $76.24 \pm 2.00$ \\
DNS                & $67.95 \pm 0.15$        & $78.14 \pm 0.19$       & $5.62 \pm 0.19$   & $1.63 \pm 0.01$      & $4028.4 \pm 130.7$ & $111.01 \pm 0.99$ \\
DNS-G              & $78.36 \pm 0.37$        & $92.64 \pm 0.39$       & $2.80 \pm 0.53$   & $1.36 \pm 0.02$      & $2384.2 \pm 213.8$ & $106.38 \pm 1.51$ \\
NSLC               & $58.64 \pm 0.52$        & $62.22 \pm 0.89$       & $0.58 \pm 0.11$   & $1.13 \pm 0.01$      & $347.5 \pm 67.9$ & $66.20 \pm 0.53$ \\
\addlinespace

\multicolumn{7}{c}{\textbf{Medium $(d=8)$}} \\
\midrule
\multicolumn{7}{c}{\textbf{MOO Methods}} \\
SoM                & $66.82 \pm 0.11$        & $87.64 \pm 0.85$       & $66.10 \pm 0.22$   & $8.65 \pm 0.06$      & $46017.1 \pm 508.0$ & $578.21 \pm 4.28$ \\
TCH-Set            & $58.86 \pm 0.06$        & $62.96 \pm 0.54$       & $3.08 \pm 0.17$   & $1.13 \pm 0.01$      & $1066.7 \pm 51.6$ & $66.55 \pm 0.29$ \\
SSoM               & $70.18 \pm 0.08$        & $88.58 \pm 0.70$       & $67.51 \pm 0.15$   & $9.25 \pm 0.04$      & $49046.0 \pm 416.1$ & $648.84 \pm 2.70$ \\
STCH-Set           & $70.17 \pm 0.09$        & $88.31 \pm 0.78$       & $67.39 \pm 0.20$   & $9.24 \pm 0.06$      & $48959.8 \pm 708.9$ & $648.40 \pm 4.32$ \\
\addlinespace
\cmidrule(r){1-7}
\multicolumn{7}{c}{\textbf{QD Methods}} \\

SQUAD              & $69.41 \pm 0.06$        & $87.70 \pm 1.13$       & $69.76 \pm 0.18$   & $9.17 \pm 0.05$      & $50030.1 \pm 620.5$ & $636.21 \pm 3.61$ \\
CMA-MAEGA          & $61.80 \pm 0.38$        & $84.41 \pm 1.51$       & $99.30 \pm 0.29$   & $9.25 \pm 0.08$      & $67465.7 \pm 472.7$ & $571.49 \pm 6.38$ \\
CMA-MEGA           & $61.77 \pm 1.44$        & $84.90 \pm 1.74$       & $97.32 \pm 0.18$   & $7.34 \pm 0.55$      & $65520.4 \pm 2830.5$ & $453.84 \pm 41.70$ \\
CMA-MAE            & $66.49 \pm 4.48$        & $77.04 \pm 0.41$       & $0.80 \pm 0.29$   & $1.25 \pm 0.06$      & $578.3 \pm 142.9$ & $82.80 \pm 2.80$ \\
PGA-ME             & $70.09 \pm 2.35$        & $78.99 \pm 0.52$       & $0.10 \pm 0.00$   & $1.07 \pm 0.01$      & $79.0 \pm 0.5$ & $75.18 \pm 2.11$ \\
DNS                & $66.89 \pm 0.23$        & $77.53 \pm 0.39$       & $5.45 \pm 0.26$   & $1.68 \pm 0.01$      & $3849.0 \pm 179.6$ & $112.04 \pm 0.56$ \\
DNS-G              & $76.02 \pm 0.28$        & $91.67 \pm 0.65$       & $2.67 \pm 0.47$   & $1.43 \pm 0.01$      & $2261.5 \pm 213.5$ & $108.60 \pm 1.12$ \\
NSLC               & $62.80 \pm 0.90$        & $64.57 \pm 0.53$       & $0.12 \pm 0.02$   & $1.07 \pm 0.02$      & $77.3 \pm 27.0$ & $66.87 \pm 0.89$ \\
\addlinespace

\multicolumn{7}{c}{\textbf{Hard $(d=16)$}} \\
\midrule
\multicolumn{7}{c}{\textbf{MOO Methods}} \\
SoM                & $70.71 \pm 0.23$        & $86.49 \pm 1.26$       & $68.51 \pm 0.12$   & $7.97 \pm 0.03$      & $50474.5 \pm 547.3$ & $563.51 \pm 3.12$ \\
TCH-Set            & $57.19 \pm 0.11$        & $62.62 \pm 0.32$       & $33.08 \pm 0.45$   & $1.34 \pm 0.01$      & $18690.7 \pm 461.7$ & $76.44 \pm 0.49$ \\
SSoM               & $73.51 \pm 0.07$        & $87.05 \pm 0.83$       & $68.51 \pm 0.13$   & $7.94 \pm 0.03$      & $52047.9 \pm 616.8$ & $583.63 \pm 2.60$ \\
STCH-Set           & $73.51 \pm 0.09$        & $86.92 \pm 1.24$       & $68.24 \pm 0.12$   & $7.94 \pm 0.03$      & $51835.9 \pm 698.0$ & $584.00 \pm 2.14$ \\
\addlinespace
\cmidrule(r){1-7}
\multicolumn{7}{c}{\textbf{QD Methods}} \\

SQUAD              & $72.86 \pm 0.04$        & $83.92 \pm 0.67$       & $68.40 \pm 0.13$   & $6.61 \pm 0.02$      & $51389.9 \pm 1014.0$ & $481.40 \pm 1.62$ \\
CMA-MAEGA          & $66.33 \pm 7.54$        & $82.46 \pm 3.26$       & $68.55 \pm 0.38$   & $4.74 \pm 2.50$      & $46340.9 \pm 31605.7$ & $297.96 \pm 146.84$ \\
CMA-MEGA           & $58.93 \pm 1.95$        & $76.88 \pm 1.75$       & $97.52 \pm 0.46$   & $3.95 \pm 0.48$      & $62526.9 \pm 2237.8$ & $233.58 \pm 36.22$ \\
CMA-MAE            & $65.04 \pm 2.66$        & $78.87 \pm 0.34$       & $2.77 \pm 0.26$   & $1.33 \pm 0.04$      & $1693.5 \pm 580.5$ & $86.38 \pm 2.28$ \\
PGA-ME             & $78.60 \pm 2.06$        & $89.48 \pm 0.55$       & $1.72 \pm 0.20$   & $1.08 \pm 0.02$      & $1397.1 \pm 754.7$ & $84.49 \pm 1.74$ \\
DNS                & $66.12 \pm 0.17$        & $77.21 \pm 0.28$       & $38.69 \pm 0.63$   & $1.68 \pm 0.02$      & $27011.8 \pm 623.1$ & $110.83 \pm 1.19$ \\
DNS-G              & $74.37 \pm 0.32$        & $90.39 \pm 0.36$       & $34.82 \pm 0.53$   & $1.45 \pm 0.01$      & $27309.3 \pm 1021.0$ & $108.02 \pm 1.08$ \\
NSLC               & $62.40 \pm 0.60$        & $64.72 \pm 0.54$       & $1.00 \pm 0.12$   & $1.09 \pm 0.01$      & $637.3 \pm 251.6$ & $67.89 \pm 0.44$ \\
\bottomrule
\end{tabular}
\end{table*}

Full results on the linear projection (LP) benchmark are reported in Table~\ref{table_lp}.

\end{document}